\documentclass[10pt,reqno]{amsart} 
\usepackage[a4paper]{geometry}

\usepackage{amsmath,amsfonts,amssymb}
\usepackage{paralist,calrsfs,dsfont}
\usepackage{graphicx}
\usepackage[hang,center,nooneline]{subfigure}

\usepackage{bm}
\usepackage{natbib}

\usepackage[plain,noend]{algorithm2e}

\theoremstyle{plain} 
\newtheorem{theorem}{Theorem}[section]
\newtheorem{lemma}[theorem]{Lemma}%[section] 
\newtheorem{corollary}[theorem]{Corollary}%[section]
\theoremstyle{definition}
\theoremstyle{remark}
\newtheorem*{notation*}{Notation}

\makeatletter
\renewcommand{\algocf@captiontext}[2]{#1\algocf@typo. \AlCapFnt{}#2} % text of caption
\def\@algocf@capt@plain{top}
\renewcommand{\algocf@makecaption}[2]{%
  \addtolength{\hsize}{\algomargin}%
  \sbox\@tempboxa{\algocf@captiontext{#1}{#2}}%
  \ifdim\wd\@tempboxa >\hsize%     % if caption is longer than a line
    \hskip .5\algomargin%
    \parbox[t]{\hsize}{\algocf@captiontext{#1}{#2}}% then caption is not centered
  \else%
    \global\@minipagefalse%
    \hbox to\hsize{\box\@tempboxa}% else caption is centered
  \fi%
  \addtolength{\hsize}{-\algomargin}%
}
\makeatother

\usepackage{color}
\definecolor{mblue}{rgb}{0,0.4470,0.7410}
\definecolor{morange}{rgb}{0.8500,0.3250,0.0980}
\definecolor{myellow}{rgb}{0.9290,0.6940,0.1250}

\newcommand{\cal}{\mathcal}
\newcommand{\Sgrp}{{\cal S}}
\newcommand{\eps}{\varepsilon}
\newcommand{\reals}{\mathbb{R}}
\newcommand{\prob}{\operatorname{P}}%{\operatorname{pr}}%{\mathds{P}}
\newcommand{\expec}{\operatorname{E}}%{\mathds{E}}

\newcommand{\rank}{\operatorname{rank}}

\newcommand{\argmax}{\operatornamewithlimits{argmax}}

\renewcommand{\le}{\leqslant}
\renewcommand{\geq}{\geqslant}
\renewcommand{\leq}{\leqslant}

\newcommand{\Move}{\mathrm{M}}
\newcommand{\diff}{\mathrm{d}}
\newcommand{\FGM}{\textsc{fgm}}
\newcommand{\1}{\mathds{1}}
\newcommand{\indic}{\1}
\newcommand{\LI}{\operatorname{LI}}

\newcounter{algocounter}

\newenvironment{algo}{
\refstepcounter{algocounter}
\begin{itemize}
\item[]
  \textsc{------------------------------ Algorithm \thealgocounter \\}
}{\mbox{}\hfill$\Box$ \end{itemize}}

\begin{document}

\iffalse
\jname{Biometrika}
%% The year, volume, and number are determined on publication
\jyear{2016}
\jvol{103}
\jnum{1}
%% The \doi{...} and \accessdate commands are used by the production team
%\doi{10.1093/biomet/asm023}
\accessdate{Advance Access publication on 31 July 2016}

%% These dates are usually set by the production team
\received{April 2012}
\revised{October 2015}
\fi

%% The left and right page headers are defined here:
% \markboth{S. Guillotte, F. Perron \and J. Segers}{Bayesian inference for bivariate ranks}

%% Here are the title, author names and addresses
\title{Bayesian inference for bivariate ranks}

\author{Simon Guillotte \and Fran\c{c}ois Perron \and Johan Segers}

\address{D\'epartement de math\'ematiques, Universit\'e du Qu\'ebec \`a Montr\'eal, 201, avenue Pr\'esident-Kennedy, Montr\'eal (Qu\'ebec) H3X 2Y7, Canada}
\email{guillotte.simon@uqam.ca}

\address{D\'epartement de math\'emathiques et de statistique, Universit\'e de Montr\'eal, Pavillon Andr\'e-Aisenstadt 2920, chemin de la Tour, Montr\'eal (Qu\'ebec) H3T 1J4, Canada}
\email{perronf@dms.umontreal.ca}

\address{Institut de statistique, biostatistique et sciences actuarielles (ISBA), Universit\'e catholique de Louvain, Voie du Roman Pays 20, B-1348 Louvain-la-Neuve, Belgium}
\email{johan.segers@uclouvain.be}
\date{\today}

\begin{abstract}
A recommender system based on ranks is proposed, where an expert's ranking of a set of objects and a user's ranking of a subset of those objects are combined to make a prediction of the user's ranking of all objects. The rankings are assumed to be induced by latent continuous variables corresponding to the grades assigned by the expert and the user to the objects. The dependence between the expert and user grades is modelled by a copula in some parametric family. Given a prior distribution on the copula parameter, the user's complete ranking is predicted by the mode of the posterior predictive distribution of the user's complete ranking conditional on the expert's complete and the user's incomplete rankings. Various Markov chain Monte-Carlo algorithms are proposed to approximate the predictive distribution or only its mode. The predictive distribution can be obtained exactly for the Farlie--Gumbel--Morgenstern copula family, providing a benchmark for the approximation accuracy of the algorithms. The method is applied to the MovieLens 100k dataset with a Gaussian copula modelling dependence between the expert's and user's grades.

\end{abstract}

\maketitle

\section{Introduction}

Recommender systems are part of many online businesses such as Amazon, e-Bay, Netflix, and others. They are tools to learn the interests of customers in order to make customer-specific recommendations of other products. These systems provide successful and valuable marketing strategies, especially due to the expansion of the world wide web and e-commerce. In 2006, Netflix organized the \textit{Netflix-Prize}, awarding one million dollars for the best algorithm. The prize was won in 2009 by a team of researchers called \textit{Bellkor’s Pragmatic Chaos} (AT\&T Labs) after over three years of competition. The problem has attracted attention in the statistical community, with statisticians working on similar problems, see for instance \cite*{feuerverger12}, \cite*{FV86}, \cite*{Lebanon12}, \cite*{Zhu14}, and the references therein.

We consider a version of recommender systems where an expert opinion ranking is available and is used, together with  a partial ranking by a costumer, in order to predict that customer's complete ranking. Essentially, we want to predict an individual's ranking of a set of $n\geq 1$ different objects, given an expert opinion ranking of the same objects. More precisely, a set of objects indexed by ${\cal N}=\{1,\ldots,n\}$ is to be evaluated and ranked by both, an expert and an individual. The expert ranks all the objects, while the individual ranks only the subset of objects corresponding to the indices in the set ${\cal M} = \{i_1, \ldots, i_m\} \subset {\cal N}$. This can happen for instance if the individual does not have knowledge yet of the objects with indices in ${\cal N} \setminus {\cal M}$.

Assume ties are impossible. Let $\Sgrp_k$ be the permutation group on the set $\{1, \ldots, k\}$. The experiment provides a complete expert's ranking $r_x=(r_x(1),\ldots,r_x(n)) \in \Sgrp_n$ as well as an incomplete user's ranking $r_y^*=(r_y^*(1),\ldots,r_y^*(m)) \in \Sgrp_m$, where $r_y^*(j)$ is the user's rank of object $i_j$ among the $m$ objects $i_1,\ldots,i_m$. The choice for the subscripts $x$ and $y$ is clarified by the following. We think of the ranks as being induced by ratings or grades measured on a continuous scale: if $x_1,\ldots,x_n$ denote the expert's grades and if $y_{i_1},\ldots,y_{i_m}$ denote the individual's grades, then $r_x = \rank(x_1,\ldots,x_n)$ and $r_y^* = \rank(y_{i_1},\ldots,y_{i_m})$. 

If the user had been able to grade all objects, the user's grades would have been $y_1,\ldots,y_n$, with corresponding ranking $r_y=\rank(y_1,\ldots,y_n)$. In view of this, the model is constructed by assuming an underlying set of latent pairs of grades $(x_1,y_1), \ldots ,(x_n,y_n)$ of all $n$ objects attributed by both the expert and the individual. Concretely, we let $(x_i,y_i)$, $i=1,\ldots,n$, be realizations of independent random vectors $(X_i,Y_i)$, $i=1,\ldots,n$, each of which is distributed according to the same bivariate distribution with continuous margins. The continuity assumption makes the marginal distributions of the grades irrelevant to the rankings. We assume that the copula of the joint distribution belongs to some parametric family indexed by a parameter $\theta \in \Theta$ for which we select a prior. Let $R_X = \rank(X_1,\ldots,X_n)$ and $R_Y = \rank(Y_1,\ldots,Y_n)$ denote the random expert and individual rankings, respectively, of the objects in ${\cal N}$, and let $R_Y^* = \rank(Y_{i_1}, \ldots, Y_{i_m})$ denote the random user ranking of the objects in ${\cal M}$. Note that $R_X$ and $R_Y$ are random elements in $\Sgrp_n$, while $R_Y^*$ is a random element in $\Sgrp_m$. The prediction of the user's complete ranking of all $n$ objects is based on the mode of the posterior predictive distribution:
\begin{equation}
\label{eq:r2hat}
  \hat{r}_y = \argmax_{r_y} \prob(R_Y = r_y \mid R_X = r_x, R_Y^* = r_y^*).
\end{equation}
This predicted ranking is then used for instance to recommend new products to the customer.

One difficulty here is the evaluation of the joint probability mass function of the pair of rankings $(R_X, R_Y)$: given a parameter value $\theta \in \Theta$, we need to compute
\begin{equation}
\label{eq:rl}
  \prob_\theta(R_X = r_x, \, R_Y = r_y), \qquad r_x, r_y \in \Sgrp_n.
\end{equation}
In most cases, this probability is not analytically tractable. In the literature, it has been referred to as 
the rank likelihood; see for instance \cite{Hoff07}, \cite*{Hoff14} and \cite*{Segers14}. The 
continuity assumption on the marginal cumulative distribution functions makes the margins irrelevant to the evaluation of the probability in \eqref{eq:rl}. The assumption of uniform margins implies that the 
probability~\eqref{eq:rl} can be evaluated by means of an integral over $[0,1]^{2n}$. Within a Bayesian approach, it also means that we need only put a prior on the copula parameter. 

An objective of this work is to find a family of copulas for which a closed-form expression of the posterior predictive distribution in~\eqref{eq:rl} is available. We will show that the Farlie--Gumbel--Morgenstern (\FGM) family \citep[p.~77]{Nelsen06} satisfies this requirement.

Since the range of dependence that can be modelled by the \FGM\ family is rather restricted, it is natural to ask how to proceed for other parametric copula families, when no explicit formulas for \eqref{eq:rl} exist. We develop a stochastic algorithm to compute \eqref{eq:rl} and we assess its accuracy by comparing its output to the results obtained from the exact formulas available for the \FGM\ family.

Another problem for computing the prediction in~\eqref{eq:r2hat} is that the cardinality of the set of rankings $r_y \in \Sgrp_n$ that are compatible with the observed ranking $r_y^* \in \Sgrp_m$ is equal to $n!/m!$. This number will usually be so high that it is infeasible to find the maximum in \eqref{eq:r2hat} by computing the probabilities on the right-hand side on \eqref{eq:r2hat} for all possible $r_y$. We will instead propose a solution based on an ergodic Monte Carlo Markov chain with the correct limiting distribution. The algorithm is applied to predict user rankings in the \textit{MovieLens 100k} dataset with a Gaussian copula modelling dependence between expert and user grades.

% =======================
\section{Rank likelihood}
\label{S:RL}

Let $\Sgrp_n$ be the permutation group of the set ${\cal N} = \{1,\ldots,n\}$. A permutation $\sigma \in \Sgrp_n$ is a bijection from ${\cal N}$ to itself; notation $\sigma = (\sigma(1),\ldots,\sigma(n))$. The group operation, denoted by $\circ$, is the usual composition of functions, that is, $\sigma\circ\tau(i)=\sigma(\tau(i))$ for $\sigma$ and $\tau$ in $\Sgrp_n$ and $i \in {\cal N}$. The group's identity element is the identity map, $e=(1,\ldots,n)$. The inverse of a permutation $\sigma\in \Sgrp_n$ is denoted by $\sigma^{-1}$ and satisfies $\sigma \circ \sigma^{-1} = e = \sigma^{-1} \circ \sigma$. 
% Finally, if $\sigma \in \Sgrp_n$ and 
% $x\in \reals^n$, let $\sigma x=(x_{\sigma(1)},\ldots,x_{\sigma(n)})$, here $\sigma$ is said to act on $x$ by 
% permutation of coordonates.

Let $\mathbb{D}_n = \{ x \in \reals^n \colon x_{(1)} < \cdots < x_{(n)} \}$ be the set of vectors in $\reals^n$ having no ties. The rank vector or ranking $\rank(x) = r_x = (r_x(1), \ldots, r_x(n))$ associated to $x \in \mathbb{D}_n$ is defined by
% \begin{align*}
%   r_x &= (r_x(1), \ldots, r_x(n)), \\
\[
  r_x(i) 
%   &= 
  =
  \sum_{j \in {\cal N}} \1(x_i \le x_j), \qquad i \in {\cal N}.
\]
% \end{align*}
We have $r_x \in \Sgrp_n$ for all $x \in \mathbb{D}_n$. We also define $r_x = \rank(x) = e$ if $x \in \reals^n \setminus \mathbb{D}_n$, ensuring that the map $\rank : \reals^n \to \Sgrp_n$ is well-defined. A simple but useful property is that the rank map behaves well under composition with permutations: for $x \in \mathbb{D}_n$ and $\sigma \in \Sgrp_n$, we have
\begin{equation}
\label{eq:rankComp}
  \rank( x_{\sigma(1)}, \ldots, x_{\sigma(n)} )
  =
  \rank( x_1, \ldots, x_n ) \circ \sigma. 
\end{equation}
% \begin{lemma}
% \label{lem:rankComp}
% For $x \in \mathbb{D}_n$ and $\sigma \in \Sgrp_n$, we have
% \[
%   \rank( x_{\sigma(1)}, \ldots, x_{\sigma(n)} )
%   =
%   \rank( x_1, \ldots, x_n ) \circ \sigma.
% \]
% \end{lemma}

Given two grading vectors $x, y \in \mathbb{D}_n$, we want to investigate the alignment, or the lack thereof, of the associated rankings $r_x = \rank(x)$ and $r_y = \rank(y)$. We would like to know the rank, under $y$, of the object that was attributed rank $j \in {\cal N}$ under the grading $x$. The original index of this object is equal to $i = r_x^{-1}(j)$, and its rank under $y$ is equal to $r_y(i) = r_y(r_x^{-1}(j))$. This leads us to the study of the permutation 
\begin{equation} 
\label{eq:ryrxinv}
  s = r_y \circ r_x^{-1} \in {\cal S}_n. 
\end{equation}
If $s = e$, for instance, the gradings $x$ and $y$ are perfectly aligned and induce the same ranking of the $n$ objects.

\iffalse
Formally, the following statistics will be useful:
\begin{equation}
\label{eq:RS}
R(x,y)=
\begin{cases}
(r_x,r_y)
&\text{if } (x,y)\in \mathbb{D}_n^2,\\
(e,e)&\text{elsewhere,}
\end{cases}
\text{ and}\quad
S(x,y)=
\begin{cases}
r_y\circ r_x^{-1}&\text{if } (x,y)\in \mathbb{D}_n^2,\\
e
&\text{elsewhere,}
\end{cases}
\end{equation}
Note that $S(x, y) = s \in \Sgrp_n$ if and only if there exists $r \in \Sgrp_n$ such that $R(x, y) = (r, s \circ r)$: 
\[
  \{ (x,y) \in \reals^{2n} : S(x,y) = s \}
  =
  \bigcup_{r \in \Sgrp_n} \{ (x,y) \in \reals^{2n} : R(x,y) = (r, s \circ r) \}.
\]
The statistic $S(x,y)$ thus represents the ranks of $y$ in the order of the ranks of $x$. 
\fi

Recall from the introduction that the random pairs $(X_i,Y_i)$, 
$i=1,\ldots,n$, represent the expert's together with the individual's gradings of $n$ objects. Assume that $(X_i,Y_i)$, 
$i=1,\ldots,n$, are independent and identically distributed (iid) random pairs with continuous margins. With probability one, there are no ties among the gradings. Consider the random rank vectors
\begin{align*}
  R_X &= \rank(X_1, \ldots, X_n), &
  R_Y &= \rank(Y_1, \ldots, Y_n).
\end{align*}
As in \eqref{eq:ryrxinv}, we want to express the ranking induced by the random grading vector $Y = (Y_1, \ldots, Y_n)$ in terms of the one induced by the random grading vector $X = (X_1, \ldots, X_n)$. This motivates the definition of the rank statistic
\[
  S(X, Y) = R_Y \circ R_X^{-1}.
\]
The joint distribution of $(R_X, R_Y)$ is determined by the distribution of $S(X, Y)$.

\begin{lemma}
\label{lem:S}
If $(X_i, Y_i)$, $i = 1, \ldots, n$, are iid random pairs with continuous margins, then
\[
  \prob( R_X = r_x, \, R_Y = r_y )
  =
  \frac{1}{n!} \prob\{ S(X, Y) = r_y \circ r_x^{-1} \},
  \qquad r_x, r_y \in \Sgrp_n.
\]
\end{lemma}

The proof of Lemma~\ref{lem:S} and of the other results in this paper are deferred to the Appendix. Let $H$ and $F, G$ be the joint and marginal cumulative distribution functions, respectively, of the random pairs $(X_i, Y_i)$, i.e.,
\begin{align*}
  H(x, y) &= \prob(X_1 \le x, Y_1 \le y), &
  F(x) &= H(x, \infty), &
  G(y) &= H(\infty, y),
\end{align*}
for $x, y \in \reals$. By assumption, $F$ and $G$ are continuous. Let $U_i = F(X_i)$ and $V_i = G(Y_i)$ for $i = 1, \ldots, n$. The random variables $U_i$ and $V_i$ are uniformly distributed on $(0, 1)$ and their joint cumulative distribution function is a copula,
\[
  C(u, v) = \prob( U_1 \le u, \, V_1 \le v ), \qquad (u, v) \in [0, 1]^2.
\]
Sklar's Theorem says that $H$ admits the representation
\begin{equation}
\label{eq:H}
  H(x,y) = C(F(x), G(y)), \qquad (x,y) \in \reals^2.
\end{equation}
With probability one, the rankings induced by the random vectors $X = (X_1, \ldots, X_n)$ and $Y = (Y_1, \ldots, Y_n)$ are the same as the ones induced by the random vectors $U = (U_1, \ldots, U_n)$ and $(V_1, \ldots, V_n)$, respectively. Since the pairs $(U_i, V_i)$, $i = 1, \ldots, n$, are iid with cumulative distribution function given by the copula $C$, it follows that the joint distribution of the rank vectors $(R_X, R_Y)$ is determined by $C$. This is formalized by the next theorem. In view of Lemma~\ref{lem:S}, it suffices to study the distribution of $S(X, Y)$.

\begin{theorem}
\label{thm:RL}
Let $(X_i,Y_i)$, $i=1,\ldots,n$, be iid random vectors with continuous margins and copula $C$. If the copula $C$ has density $c$, then we have
\begin{equation}
 \label{eq:RL}
 \prob\{S(X,Y)=s\}=\frac{1}{n!}\expec\left\{ \prod_{i=1}^n c\Bigl({\textstyle \sum_{j=1}^i
W_{1,j},\sum_{j=1}^{s(i)} 
W_{2,j}}\Bigr)\right\},\qquad s\in \Sgrp_n,
\end{equation}
where $(W_{\ell,1},\ldots,W_{\ell,n+1})=(W_{\ell,1},\ldots,W_{\ell,n},1-\sum_{j=1}^nW_{\ell,j})$, $\ell=1,2$, are iid according to the $\operatorname{Dirichlet}(1,\ldots,1)$ distribution. 
\end{theorem}

Among other things, Theorem~\ref{thm:RL} shows the intuitively obvious property that the marginal distributions of the gradings do not affect the joint distribution of the rank vectors. We shall therefore assume that the marginal distributions of the expert's and user's grades are both uniform on $(0,1)$. Then we have $X_i = U_i$ and $Y_i = V_i$, and the random pairs $(X_i, Y_i)$, $i = 1, \ldots, n$, are independent and identically distributed according to the copula $C$. This assumption also allows us to unambiguously write $S = S(X, Y) = S(U, V)$.

When the copula function belongs to a parametric family $(C_{\theta} : \theta \in \Theta)$, the distribution of $S$ depends on $\theta$. The probability mass function of $S$, seen as a function of $\theta$, i.e., the map $\theta \mapsto \prob_\theta(S = s)$, for $s \in \Sgrp_n$, is sometimes referred to as the rank likelihood. The expression~\eqref{eq:RL} is particularly helpful as it suggests a certain Monte Carlo algorithm to compute this rank likelihood; see Algorithm~\ref{S:Algo1}.

% ======================================================
\section{Predictive distribution of compatible rankings}
\label{S:CR}

% ------------------------------
\subsection{Compatible rankings}

In the predictive distribution \eqref{eq:r2hat}, any candidate complete ranking $r_y \in \Sgrp_n$ by the user should be compatible with the observed ranking $r_y^* \in \Sgrp_m$ of the $m$ objects graded by the user. The notion of compatible rankings has already appeared in the literature, see for instance~\cite{Alvo14}. Recall that $\Sgrp_n$ is the group of permutations of the set ${\cal N}=\{1,\ldots,n\}$. 

An incomplete ranking of size $m$, with $m \in \{1, \ldots, n-1\}$, is a couple $(r^*, {\cal M})$ consisting of a permutation $r^*\in \Sgrp_m$ and a subset ${\cal M}=\{i_1,\ldots,i_m\} \subset {\cal N}$, with $1\leq i_1 < \cdots < i_m \leq n$. For example, for $n=4$, an incomplete ranking of size $m=3$ is given by $r^*=(3,1,2)$ and ${\cal M} = \{1,3,4\}$. This incomplete ranking corresponds to the partial ranking $r=(3,-,1,2)$, the second object not (yet) being ranked among the three other ones.

The set $\mathcal{C}(r^*, {\cal M})$ of compatible rankings associated to the incomplete ranking $(r^*, {\cal M})$ is defined as the set of rankings $r \in \Sgrp_n$ of all $n$ objects such that the ranking of the $m$ objects in ${\cal M}$ induced by $r$ is equal to $r^*$. Formally, we have
\begin{eqnarray*}
  \mathcal{C}(r^*,{\cal M})
  &=& \{ r \in \Sgrp_n : \, \rank(r(i_1),\ldots,r(i_m)) = r^*\}.
  \\
  &=& \{ r \in \Sgrp_n : \, r(i_{\sigma(1)}) < r(i_{\sigma(2)}) < \cdots < r(i_{\sigma(m)}), \, \sigma = (r^*)^{-1} \}.
\end{eqnarray*}
For the example above, we have
\[
 \mathcal{C}(r^*, {\cal M})=\{(3,4,1,2),(4,3,1,2),(4,2,1,3),(4,1,2,3)\}.
\]

To select an $(r^*, {\cal M})$-compatible ranking $r$, it suffices to choose the ranks, $r(i) \in {\cal N}$, of the $n-m$ objects $i \in {\cal N} \setminus {\cal M}$. The ranks of the remaining $m$ objects in ${\cal M}$ are then determined by the compatibility constraint. This shows that the cardinality of $\mathcal{C}(r^*, {\cal M})$ is equal to $n!/m!$.  

In the original formulation of the problem, the permutations $r_x \in \Sgrp_n$ and $r_y \in \Sgrp_n$ represent the expert and the individual's complete rankings respectively. The set ${\cal M} = \{ i_1, \ldots, i_m \} \subset {\cal N}$ represents the indices of the $m$ objects ranked by the individual, with $1\leq i_1<\cdots <i_m\leq n$ and $1 \le m < n$. One observes the expert's complete ranking $r_x$ and the user's ranking $r_y^*=(r_y(i_1),\ldots,r_y(i_m)) \in \Sgrp_m$ of the objects in ${\cal M}$. Notice that if the expert would have ranked only the $m$ objects in ${\cal M}$, then the expert's ranks would have been
\[
  r_x^* = \rank( r_x(i_1), \ldots, r_x(i_m) ) \in \Sgrp_m.
\]

Further, consider the permutation
\[
  s^* = r_y^* \circ (r_x^*)^{-1} \in \Sgrp_m.
\]
In words, $s^*(j)$ is the user's rank of the object that was given rank $j = 1, \ldots, m$ by the expert, among the $m$ objects graded by the user. If $s^* = e$, the identity permutation in $\Sgrp_m$, then the rankings by the user and the expert are perfectly aligned. 

By using Lemma~\ref{lem:S}, it will be shown in~\eqref{eq:preddist} below that for the calculation of the posterior predictive distribution~\eqref{eq:r2hat}, we can simply work with the transformation $s^*$ as our observed data instead of with $r_x$, ${\cal M}$, and $r_y^*$. However, we need to translate the compatibility constraint to the transformed rankings. This is the purpose of Lemma~\ref{lem:cmpblt} below, which says that $r_y$ is compatible with $r_y^*$ for the objects in ${\cal M}$ if and only if $r_y \circ r_x^{-1}$ is compatible with $s^*$ for the objects in the set ${\cal M}^*$ defined in \eqref{eq:Mstar}. This statement can be further interpreted as if the $n$ objects were lined up in the order of the expert's preference (and so $r_x=e \in \Sgrp_n$) and the user was to rank the $m$ objects presented to him in that order: the result would be $s^*$ and ${\cal M}^*$.
 
The order statistics of the expert ranks $r_x(i_1), \ldots, r_x(i_m) \in {\cal N}$ of the $m$ objects graded by the user are denoted by 
% \[
$  1 \le i_1^* < \cdots < i_m^* \le n$. 
% \]
For the reason explained above, it will be convenient to consider the incomplete ranking $(s^*, {\cal M}^*)$ with
\begin{equation}
\label{eq:Mstar}
  {\cal M}^* 
  = \{ i_1^*, \ldots, i_m^* \} 
  = \{ r_x(i_1), \ldots, r_x(i_m) \} 
  \subset {\cal N}.
\end{equation}

To apply Lemma~\ref{lem:S}, we would like to switch from $r_y$ to $r_y \circ r_x^{-1}$. The following lemma says how this transformation affects the compatibility constraint.

\begin{lemma}
\label{lem:cmpblt}
We have
\[
  r_y \in {\cal C}( r_y^*, {\cal M} )
  \iff
  r_y \circ r_x^{-1} \in {\cal C}( s^*, {\cal M}^* ).
\]
\end{lemma}

% ----------------------------------
\subsection{Predictive distribution}

Let $(C_\theta : \theta \in \Theta)$ be parametric family of bivariate copulas and let $\pi(\theta)$, $\theta \in \Theta$, be a prior density on $\theta$. Conditionally on $\theta$, the random pairs $(X_i, Y_i)$, $i = 1, \ldots, n$, are iid with common distribution given by $C_\theta$. We observe the complete expert ranking $R_X = r_x \in \Sgrp_n$ as well as the partial user ranking $R_Y^* = r_y^* \in \Sgrp_m$ on ${\cal M} = \{i_1, \ldots, i_m\} \subset {\cal N}$ as above. The posterior predictive distribution, or predictive distribution in short, of the complete user ranking $R_Y$ given the data is
\begin{align*}
  \prob( R_Y = r_y \mid R_X = r_x, \, R_Y^* = r_y^* )
%   &= \frac{\prob( R_X = r_x, \, R_Y^* = r_y^*, \, R_Y = r_y )}{\prob( R_X = r_x, \, R_Y^* = r_y^* )} \\
  &=
  \frac%
    {\int_{\Theta} \prob_\theta( R_X = r_x, \, R_Y^* = r_y^*, \, R_Y = r_y ) \, \pi(\theta) \, \diff \theta}%
    {\int_{\Theta} \prob_\theta( R_X = r_x, \, R_Y^* = r_y^* ) \, \pi(\theta) \, \diff \theta},
\end{align*}
for $r_y \in \Sgrp_n$. For the numerator, we use Lemma~\ref{lem:S} to find that
\begin{multline*}
  \prob_\theta( R_X = r_x, \, R_Y^* = r_y^*, \, R_Y = r_y ) \\
  =
  \begin{cases}
    \prob_\theta( R_X = r_x, \, R_Y = r_y ) = \frac{1}{n!} \prob_\theta( S = r_y \circ r_x^{-1} ),
    & \text{if $r_y \in {\cal C}( r_y^*, {\cal M} )$,} \\
    0
    & \text{otherwise}.
  \end{cases}
\end{multline*}
Summing over all $r_y \in \Sgrp_n$, we find for the denominator that
\[
  \prob_\theta( R_X = r_x, \, R_Y^* = r_y^* )
  =
  \sum_{r_y' \in {\cal C}( r_y^*, {\cal M} )} \frac{1}{n!} \prob_\theta( S = r_y' \circ r_x^{-1} ).
\]
The marginal distribution of $S$ (marginal with respect to $\theta$) is
\begin{equation}
\label{eq:marginal_probs}
  \prob( S = s ) = \expec_\pi \{ \prob_\theta( S = s ) \} = \int \prob_\theta( S = s ) \, \pi(\theta) \, \diff \theta,
  \qquad s \in \Sgrp_n.
\end{equation}
As a consequence, the predictive distribution of $R_Y$ given the data is
\[
  \prob( R_Y = r_y \mid R_X = r_x, \, R_Y^* = r_y^* )
  = 
  \1\{ r_y \in {\cal C}( r_y^*, {\cal M} ) \} \, 
  \frac%
    {\prob( S = r_y \circ r_x^{-1} )}%
    {\sum_{r_y' \in {\cal C}( r_y^*, {\cal M} )} \prob( S = r_y' \circ r_x^{-1} )},
\]
for $r_Y \in \Sgrp_n$. Write $r_y = s \circ r_x$ with $s = r_y \circ r_x^{-1}$ and note from Lemma~\ref{lem:cmpblt} that $r_y \in {\cal C}( r_y^*, {\cal M} )$ if and only if $s \in {\cal C}( s^*, {\cal M}^* )$. We find that the predictive distribution of $R_Y$ given the data is
\begin{align}
\nonumber
  \prob( R_Y = s \circ r_x \mid R_X = r_x, \, R_Y^* = r_y^* )
  &=
  \1\{ s \in {\cal C}( s^*, {\cal M}^* ) \} \,
  \frac%
    {\prob( S = s )}%
    {\sum_{s' \in {\cal C}( s^*, {\cal M}^* )} \prob( S = s' )} \\
  &=
  \prob\{ S = s \mid S \in {\cal C}( s^*, {\cal M}^* ) \}
  =: p(s),
\label{eq:preddist}
\end{align}
for $s \in \Sgrp_n$. To ease the notation and terminology, we call $p(s)$, $s \in {\cal C}( s^*, {\cal M}^* )$, the predictive distribution of the compatible rankings. Since $p(s) = 0$ for $s \not\in {\cal C}( s^*, {\cal M}^* )$, we do not need to consider such permutations.

The predicted ranking, $\hat{r}_y$, for the user is equal to the mode of the predictive distribution. In view of the above identities, we have
\[
  \hat{r}_y 
  = \argmax_{r_y} \prob( R_Y = r_y \mid R_X = r_x, \, R_Y^* = r_y^* )
  = \hat{s} \circ r_x
\]
where
\begin{equation}
\label{eq:shat}
  \hat{s} 
  = \argmax_{s \in {\cal C}( s^*, {\cal M}^* )} p(s)
  = \argmax_{s \in {\cal C}( s^*, {\cal M}^* )} \prob( S = s ).
\end{equation}

Two questions arise: How to compute the marginal and predictive distributions $\Pr( S = s )$ and $p(s)$, respectively? How to find the mode, $\hat{s}$, of the predictive distribution? For the family of Farlie--Gumbel--Morgenstern (\FGM) copulas, we can find explicit formulas for the marginal probabilities $\Pr( S = s )$. For other copula families, we propose Monte Carlo algorithms, the performance of which we assess by using the explicit formulas for the \FGM\ family as benchmark.

\section{Farlie--Gumbel--Morgenstern copula family}
\label{S:FGM}

% The copulas in the Farlie--Gumbel--Morgenstern (FGM) family have the form $C_{\theta}(u,v) = uv \{1 + \theta(1-u)(1-v)\}$, with densities $c_{\theta}(u,v) = 1 + \theta(1-2u)(1-2v)$, for $(u,v) \in [0,1]^2$ and with parameter $\theta \in \Theta = [-1, 1]$. The fact that the density is polynomial allows us to evaluate the rank likelihood $\prob_\theta( S = s ) $ in \eqref{eq:RL} explicitly. 
% Given a prior, $\pi(\theta)$, on $\theta$, the marginal distribution $\Pr( S = s ) = \int \prob_\theta(S = s) \, \pi(\theta) \, \diff \theta$ depends on the prior distribution only through its moments up to order $n-1$. We compute the marginal distribution for a symmetric beta prior and for Jeffreys' prior. Conditioning on the event $\{ S \in {\cal C}( s^*, {\cal M}^* ) \}$ then yields the predictive distribution $p(s)$.

% --------------------------
\subsection{Rank likelihood}

% We compute the probability~\eqref{eq:RL} when the copula, $C$, belongs to the FGM family. 
The copulas in the \FGM\ family have the form $C_{\theta}(u,v) = uv \{1 + \theta(1-u)(1-v)\}$, with densities $c_{\theta}(u,v) = 1 + \theta(1-2u)(1-2v)$, for $(u,v) \in [0,1]^2$ and with parameter $\theta \in \Theta = [-1, 1]$. The fact that the density is polynomial allows us to evaluate the rank likelihood $\prob_\theta( S = s ) $ in \eqref{eq:RL} explicitly. The resulting expression is a polynomial of degree $n-1$ in $\theta$.

\begin{theorem}
\label{thm:FGM}
Let $(X_i, Y_i)$, $i = 1, \ldots, n$, be iid random vectors with continuous margins and \FGM\ copula $C_\theta$, $\theta \in [-1, 1]$. Then
\begin{equation}
 \label{eq:FGM}
 \prob_{\theta}(S=s)
 =
 \sum_{j=0}^{n-1} c_j(s) \, \theta^j,
 \qquad s \in \Sgrp_n,
\end{equation}
with $c_0(s)=1/n!$, and 
\begin{equation}
\label{eq:cj}
  c_j(s)
  = 
  n! \sum_{1\leq i_1<i_2<\cdots<i_j\leq n}
  d_j(i_1,\ldots,i_j) \, d_j(s(i_1),\ldots,s(i_j)), \quad j=1,\ldots,n-1,
\end{equation}
where
\begin{equation}
\label{eq:dj}
  d_j(i_1,\ldots,i_j)
  =
  \frac{1}{(n+j)!}
  \sum_{k_1=1}^{n+1} \cdots \sum_{k_j=1}^{n+1}(-1)^{\sum_{\ell=1}^j\indic(k_{\ell}>i_{\ell})}
  \prod_{p=1}^{n+1}\left\{\sum_{\ell=1}^j\indic(k_{\ell}=p)\right\}!.
\end{equation}
\end{theorem}

The \FGM\ model gives rise to some symmetries in the rank likelihood. Let $a = (n, \ldots, 1) \in \Sgrp_n$ be the anti-identity. Note that $a^{-1} = a$ and put $a^0 = e$.

\begin{lemma}
\label{lem:symmetries}
For $s \in \Sgrp_n$ and $\theta \in [-1, 1]$, we have, for the \FGM\ copula family,
\begin{equation}
\label{eq:symmetries}
  \prob_\theta( S = s )
  =
  \prob_{(-1)^{i+j}\theta}( S = a^i \circ s^k \circ a^j ),
  \qquad i,j \in \{0, 1\}, \ k \in \{-1, 1\}.
\end{equation}
\end{lemma}

Inspecting the proof of Lemma~\ref{lem:symmetries}, we see that the symmetry property \eqref{eq:symmetries} holds for any family of copula densities $(c_\theta : \theta \in \Theta)$ such that $c_\theta(u, v) = c_\theta(v, u)$ and $c_\theta(1-u, v) = c_{-\theta}(u, v)$ for all $(u, v) \in [0, 1]^2$, where it is assumed that the parameter set $\Theta \subset \reals$ is symmetric around the origin. Besides the \FGM\ family, this includes, after reparametrization, the bivariate Frank, Plackett, and Gauss copula families.

% ------------------------------------------------------
\subsection{Marginal distribution of the rank statistic}

Let the \FGM\ parameter $\theta$ have prior density $\pi$ over $\Theta=[-1,1]$. It follows from Theorem~\ref{thm:FGM} that the marginal distribution of $S$ is
\begin{equation*}
  \prob(S=s)
  = \expec_{\pi}\{\prob_{\theta}(S=s)\}
  = \int_{-1}^1 \prob_{\theta}(S=s) \, \pi(\theta) \, d\theta 
  = \sum_{j=0}^{n-1} c_j(s) \, \int_{-1}^1 \theta^j \, \pi(\theta) \, d\theta,
\end{equation*}
for $s \in \Sgrp_n$. The marginal probabilities $\prob(S = s)$ are directly obtained via the calculation of the moments of order $1, \ldots, n-1$ of the prior distribution.

The symmetries found in Lemma~\ref{lem:symmetries} for the rank likelihood carry through to the marginal distribution. If the prior on $[-1, 1]$ is symmetric, which is the case for Jeffreys' prior discussed below, we obtain the equality
\begin{equation}
 \label{eq:symeq}
 \prob(S=s)=\prob(S=a^i\circ s^k \circ a^j), \qquad i,j\in \{0,1\}, k\in \{-1,1\},
\end{equation}
with $a = (n, n-1, \ldots, 1) \in \Sgrp_n$ the anti-identity. This symmetry property reduces the number of distinct values for the probabilities $\prob(S=s)$, $s\in \Sgrp_n$, $n\geq 2$. 
% The set of permutations that 
% commute with $a$, known as the centralizer of $a$ is $C(a)=\{s:\in \Sgrp_n:a\circ 
% s=s \circ a\}$ and the size of this class is given by $|C(a)|=2^{\lfloor n/2 
% \rfloor}\lfloor n/2 \rfloor !$, see~\cite{Stanley99}, and so there are at
% most $(n!+|C(a)|)/4$ distinct marginal probabilities.

Next we investigate the mode of the marginal distribution of $S$. In many cases, the identity, $e$, or the anti-identity, $a$, are modes, and sometimes even both rankings are modal. The next result gives sufficient conditions for $e$ or $a$ to be modal rankings.

\begin{theorem}
\label{thm:mode}
Let $\pi$ be a (prior) density on the \FGM\ parameter $\theta \in \Theta = [-1, 1]$.
\begin{compactenum}[(i)]
\item If the odd order moments of $\pi$ are nonnegative, that is, if $\expec_{\pi}(\theta^{2k+1})\geq 0$, for $k=0,\ldots, \lfloor n/2 \rfloor -1$, then the identity $s=e$ is a mode of the marginal distribution $\prob(S=s)$, $s\in \Sgrp_n$.
\item If the odd order moments of $\pi$ are nonpositive, then the anti-identity $s=a$ is a mode of the marginal distribution.
\end{compactenum}
\end{theorem}

% ------------------------------
\subsection{Prior distributions}

We consider two priors on the \FGM\ family, namely the Beta distribution and Jeffreys' prior.

For the Beta prior, let $\theta=2T-1$ with $T \sim \operatorname{Beta}(\alpha,\beta)$ and parameters $\alpha>0$ and $\beta>0$,  and let $\pi_{\alpha,\beta}$ denote the resulting density. The moments of $\theta$ are easily computed:
\[
  \int_{-1}^1 \theta^n \, \pi_{\alpha,\beta}(\theta) \, \diff \theta
  =
  \frac{(-1)^n}{(n+1) B(\alpha,\beta)}
  \sum_{k=0}^n (-1)^k \frac{B(\alpha+k,\beta+n-k)}{B(1+k,1+n-k)}, 
  \qquad n = 0, 1, 2, \ldots,
\]
where $B(\alpha,\beta)$, for $\alpha>0, \beta>0$, is the beta function. We have a
corollary to Theorem~\ref{thm:mode} for this particular choice of prior.

\begin{corollary}
\label{cor:mode}
Let $\theta=2T-1$ with $T\sim \operatorname{Beta}(\alpha,\beta)$ for $\alpha > 0$ and $\beta > 0$.
\begin{compactenum}[(i)]
\item If $0<\beta \leq \alpha$, then the identity $s=e$ is a mode of the marginal distribution of $S$.
\item If $0<\alpha\leq \beta$, then the anti-identity $s=a$ is a mode of the marginal distribution of $S$.
\end{compactenum}
\end{corollary}

We now compute Jeffreys' prior $\pi_J(\theta)\varpropto \sqrt{I(\theta)}$, for $\theta 
\in [-1,1]$, with $I(\theta)$ the Fisher information at $\theta$. We have
% \begin{eqnarray*}
\[
  I(\theta)
%   &=&
  =
  \int_{(0, 1)^2}
    \left(\frac{\partial}{\partial \theta}\log c_{\theta}(u, v)\right)^2 \, 
  \diff (u, v) \\
%   &=& 
%   \int_0^1 \int_0^1
%     \frac{(2u-1)^2(2v-1)^2}{1+\theta(1-2u)(1-2v)} \,
%   \diff u \, \diff v \\
%   &=&
  =
  \begin{cases}
  \frac{1}{9}
  &\text{if $\theta=0$,} \\
  \frac{1}{\theta^2} \left\{\frac{\LI_2(\theta)-\LI_2(-\theta)}{2\theta} -1\right\}
  &\text{if $\theta \in [-1, 1] \setminus \{ 0 \}$,}
  \end{cases}
\]
% \end{eqnarray*}
where $\LI_2(x) = -\int_0^1 y^{-1} \log(1-x y) \, \diff y =   \sum_{k=1}^{\infty} k^{-2} x^k$, for $x \leq 1$, is the dilogarithm function. It follows that the Fisher information is
\[
  I(\theta) 
  = 
  \sum_{k=0}^{\infty}
  \frac{\theta^{2k}}{(2k+3)^2}, \qquad \theta \in [-1, 1].
\]
See Figure~\ref{fig:Jeff}(a) for a graph of Jeffreys' prior, $\pi_J$. Its odd order moments vanish because of symmetry, and its even order moments can be computed by numerical quadrature.

Jeffreys' prior $\pi_J$ being symmetric, we compare it with the symmetric subfamily $\pi_{\alpha,\alpha}$, $\alpha>0$, of the Beta prior. We consider the total variation distance between $\pi_J$ and $\pi_{\alpha,\alpha}$, i.e.,
\[
  \operatorname{TV}(\alpha)
  =
  \frac{1}{2}
  \int_{-1}^1 
    \lvert \pi_J(\theta) - \pi_{\alpha,\alpha}(\theta) \rvert 
  \, \diff \theta, \qquad \alpha>0.
\]
The minimal value of $\operatorname{TV}(\alpha)$, obtained numerically, is attained when $\alpha=0{\cdot}88$, with $\operatorname{TV}(0{\cdot}88)=0{\cdot}0082$, see Figure~\ref{fig:Jeff}(b). The symmetric Beta prior with this value for $\alpha$ may be a numerically tractable alternative to Jeffreys' prior.

\begin{figure}
\begin{center}
\begin{tabular}{cc}
\includegraphics[width=0.4\textwidth]{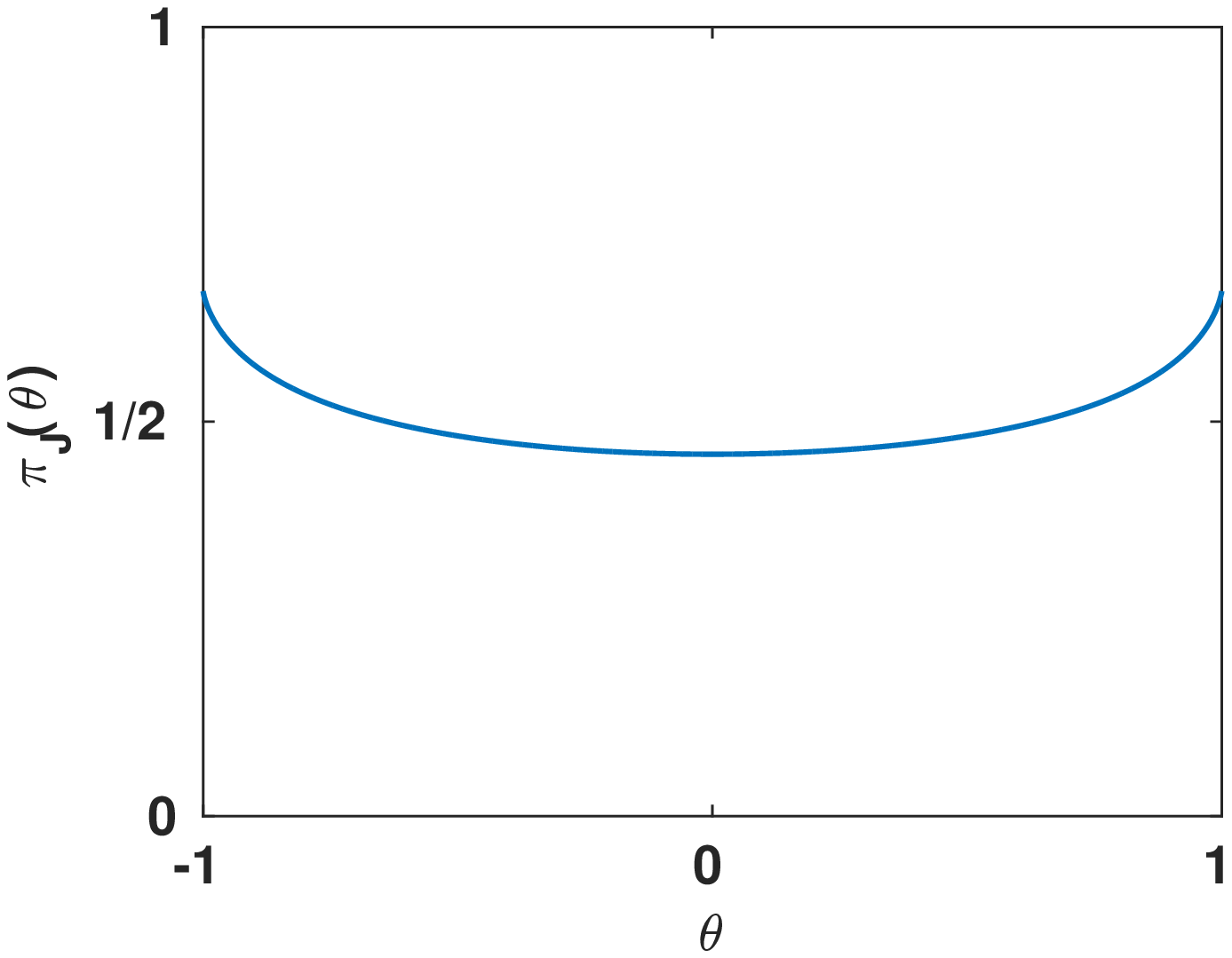}&
\includegraphics[width=0.4\textwidth]{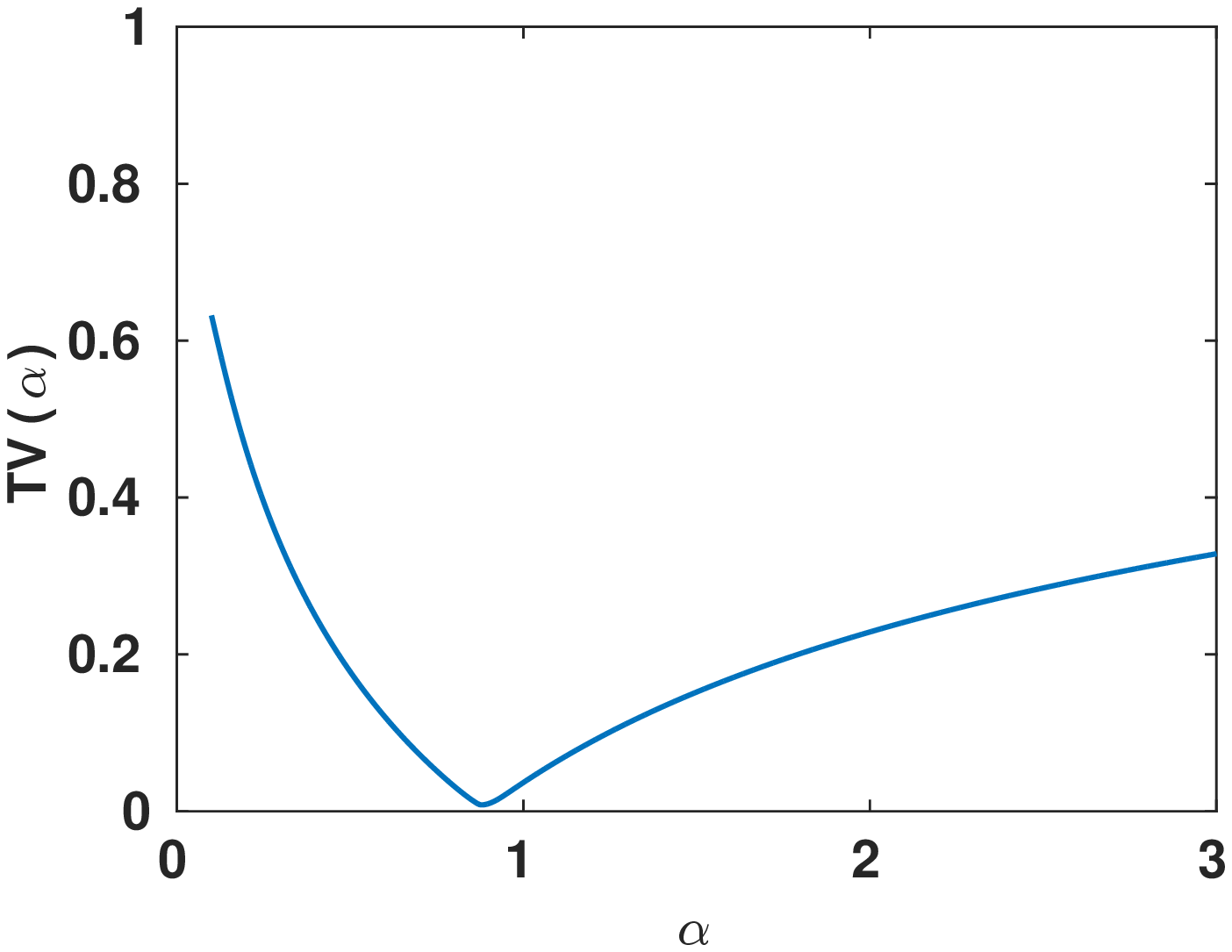}
\end{tabular}
\end{center}
% \figurebox{5cm}{}{}[JeffreysFGM]
% \figurebox{5cm}{}{}[TVJeffBeta]
\caption{Jeffreys' prior (left) for the FGM family and its total variation distance (right) to the density of the random variable $\theta=2T-1$, where $T\sim \operatorname{Beta}(\alpha,\alpha)$ and $\alpha \in [0{\cdot}1,3]$.}
\label{fig:Jeff}
\end{figure} 

We illustrate the marginal distributions $\prob(S=s)$, $s\in \Sgrp_n$, obtained 
with Jeffreys' prior and with an asymmetrical Beta prior. The lack of a universal total ordering on $\Sgrp_n$ makes graphing a bit difficult. We visualize the marginal distributions arising from both priors by plotting the marginal probabilities of $s \in \Sgrp_n$ against the Kendall distance, $d_{\tau}$, of $s$ from the modal rankings, the latter depending on the prior. The Kendall distance \citep{Diaconis88} on $\Sgrp_n$ is given essentially by the number of discordances between two permutations; more precisely,
\begin{equation}
\label{eq:dtau}
  d_{\tau}(s,s')
  =
  \sum_{1\leq i<j \leq n}
  \indic(s'\circ s^{-1}(i) > s'\circ s^{-1}(j)), 
  \qquad s,s' \in \Sgrp_n.
\end{equation}
% Another popular one in statistical problems is Spearman's footrule distance 
% $d_S$, where 
% $d_S(r,r')=\sum_{i=1}^n|r(i)-r'(i)|$. It was shown in \cite{DG77} that these two 
% metrics are equivalent. 
In particular, we have the relation $0\leq \binom{n}{2}^{-1}d_{\tau}(s,s')=\{1-\tau(s,s')\}/2\leq 1$, where $\tau$ is the sample version of Kendall's tau of the sample $(s(1), s'(1)), \ldots, (s(n), s'(n))$.

\begin{figure}
\begin{center}
\begin{tabular}{cc}
\includegraphics[width=0.4\textwidth]{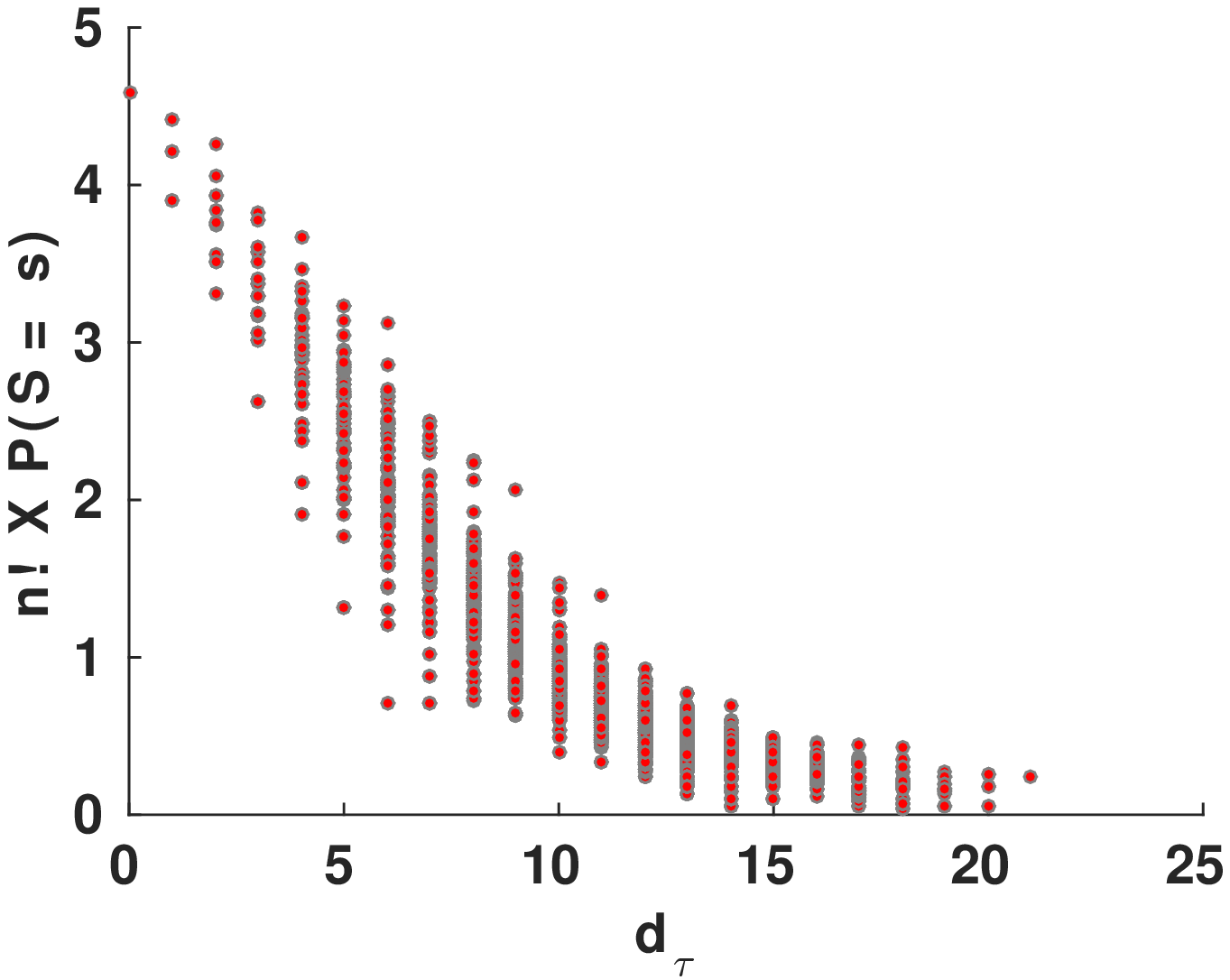}&
\includegraphics[width=0.4\textwidth]{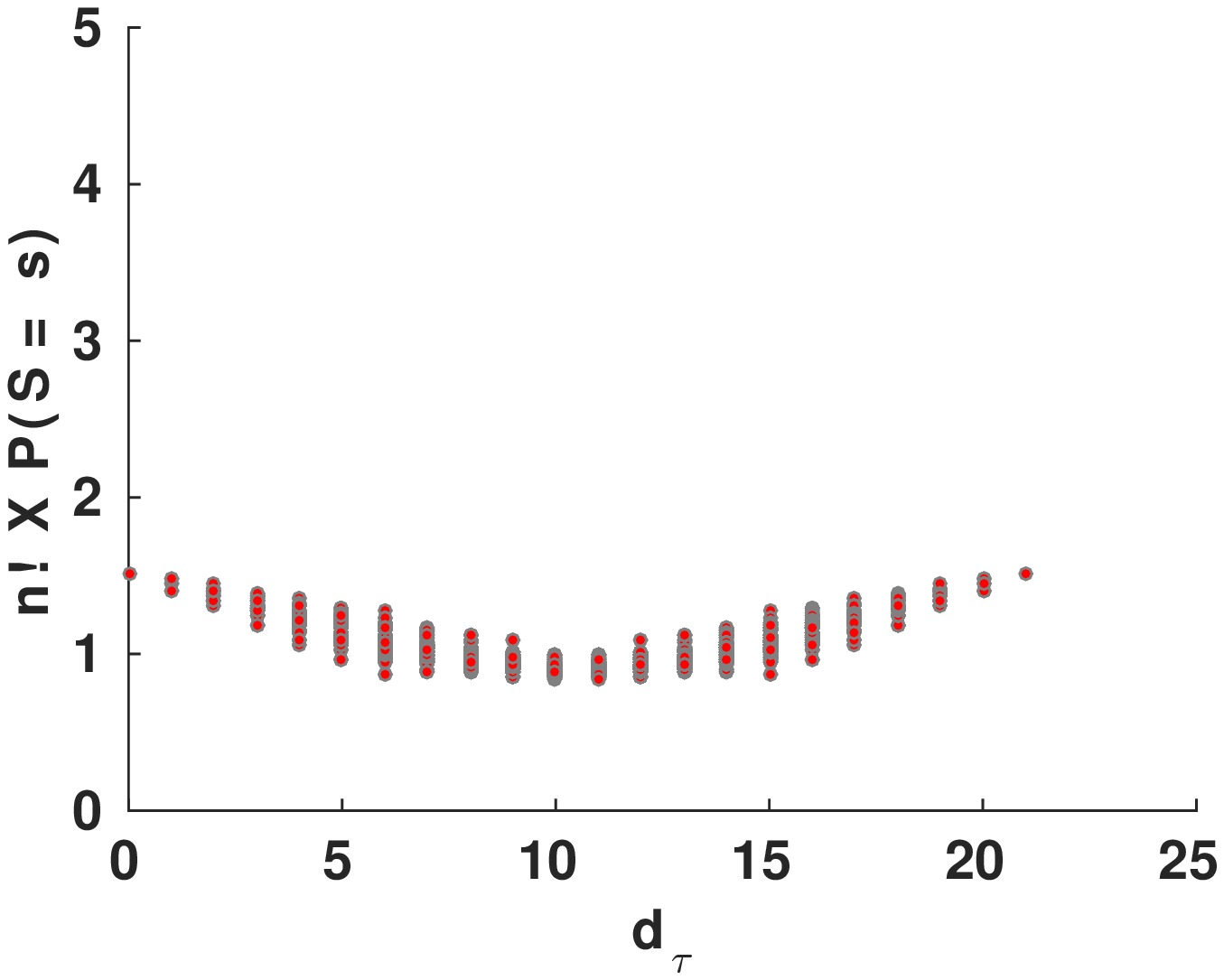}
\end{tabular}
\end{center}
% \figurebox{4cm}{}{}[fullranksAsym]
% \figurebox{4cm}{}{}[fullranksJeff]
\caption{Rescaled marginal distribution $n! \prob(S=s)$, $s\in 
S_7$, for the \FGM\ copula family, plotted against the Kendall distance $d_\tau(s, s_\pi)$ of the rankings $s$ from the modal rank $s_\pi$ for the given prior $\pi$. Left: asymmetrical Beta prior $\theta=2T-1$ with $T\sim \text{Beta}(1/10,2)$. Right: Jeffreys' prior $\pi_J$.}
\label{fig:mardist}
\end{figure}

The marginal distributions are illustrated in Figure~\ref{fig:mardist} for an asymmetric Beta prior on the left and for Jeffrey's prior on the right. The superposition of points is explained by Lemma~\ref{lem:symmetries} and equation~\eqref{eq:symeq}. For the asymmetric Beta prior, Corollary~\ref{cor:mode} implies that the mode of the marginal distribution is the anti-identity $s = a$. Jeffrey's prior is symmetric, so that, by Theorem~\ref{thm:mode}, both the identity, $s = e$, and the anti-identity, $s = a$, are modes of the marginal distribution. The symmetry that appears for Jeffrey's prior is also an artifact of the \FGM\ model and will also appear for other exchangeable and radially symmetric copula families, as discussed after Lemma~\ref{lem:symmetries}. Since $a \circ \sigma=(n+1-\sigma(1),\ldots,n+1-\sigma(n))$, with $a = (n, \ldots, 1)$ the anti-identity and $\sigma \in \Sgrp_n$, we get
\begin{equation}
\label{eq:symdist}
  d_{\tau}(s,a\circ s')
  =
  \binom{n}{2} - d_{\tau}(s,s'), \qquad s,s' \in \Sgrp_n.
\end{equation}
Together with the equality~\eqref{eq:symeq}, we obtain that the marginal probabilities are symmetrical with respect to the midrange distance $\binom{n}{2}/2$.

% ----------------------------------
\subsection{Predictive distribution}

The posterior predictive distribution, $p(s)$, of the rank statistic $S$ is equal to the marginal distribution conditioned on the event $\{ S \in {\cal C}(s^*, {\cal M}^*) \}$; see \eqref{eq:preddist}. The polynomial form of the rank-likelihood \eqref{eq:FGM} induced by the \FGM\ family together with moment formulas for the prior distributions then allow us to compute predictive probabilities $p(s)$ exactly.

To provide an example, take as a toy problem the incomplete rankings $(-,2,-,1,3,-,-)$ or in other words, $n=7$, ${\cal M}^*=\{2,4,5\}$ and $s^*=(2,1,3)$, and consider the same two priors as in Figure~\ref{fig:mardist}. The predictive distribution $p(s)$, $s \in \mathcal{C}(s^*,{\cal M}^*)$, from equation~\eqref{eq:preddist} is illustrated in Figure~\ref{fig:preddist}. The predictive distribution associated to Jeffreys' prior has two modes, $s=(1,4,2,3,5,6,7)$ and $s^{-1}=(1,3,4,2,5,6,7)$. We will return to this toy example in Section~\ref{S:Algo2}.

\begin{figure}
\begin{center}
\begin{tabular}{cc}
\includegraphics[width=0.4\textwidth]{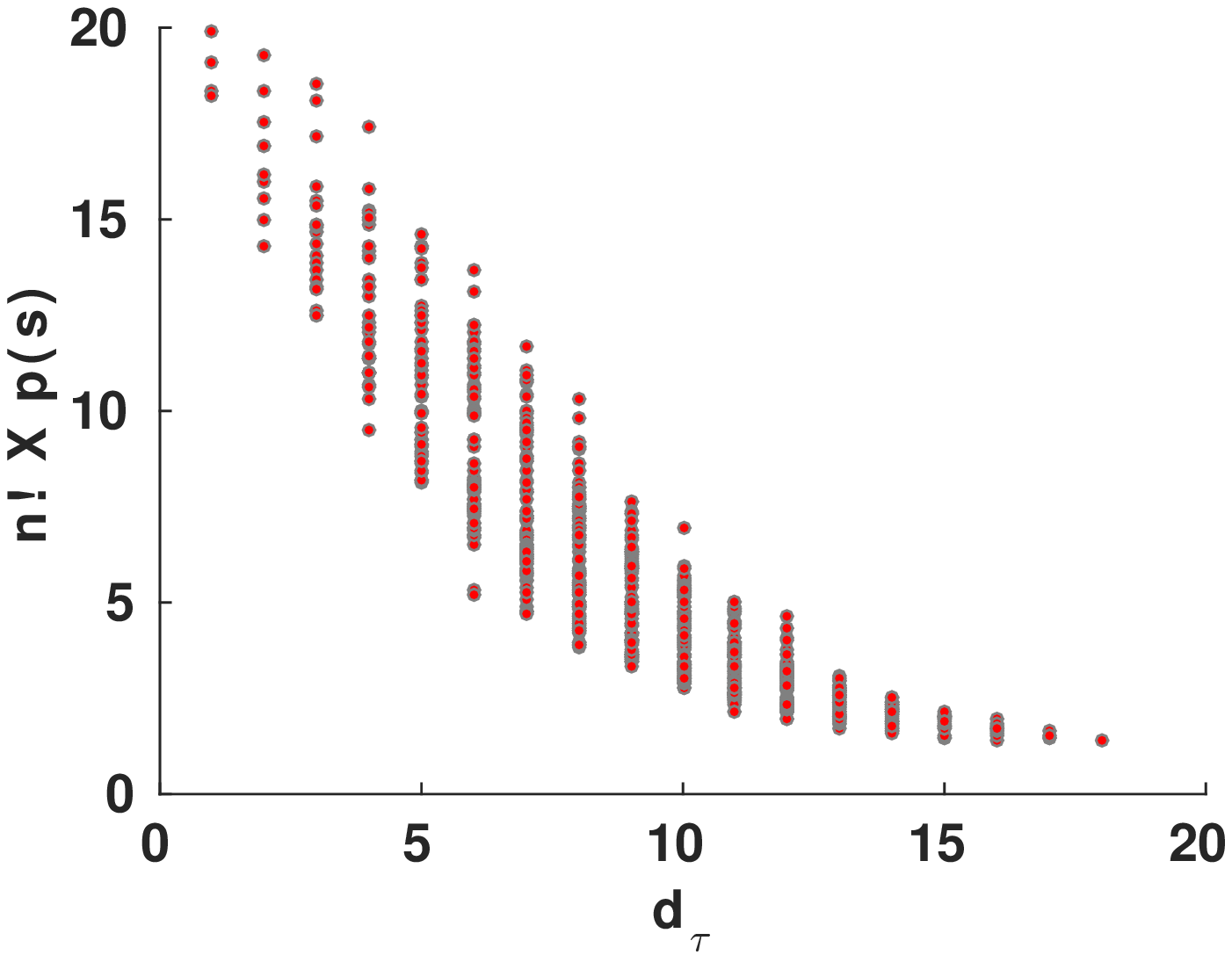}&
\includegraphics[width=0.4\textwidth]{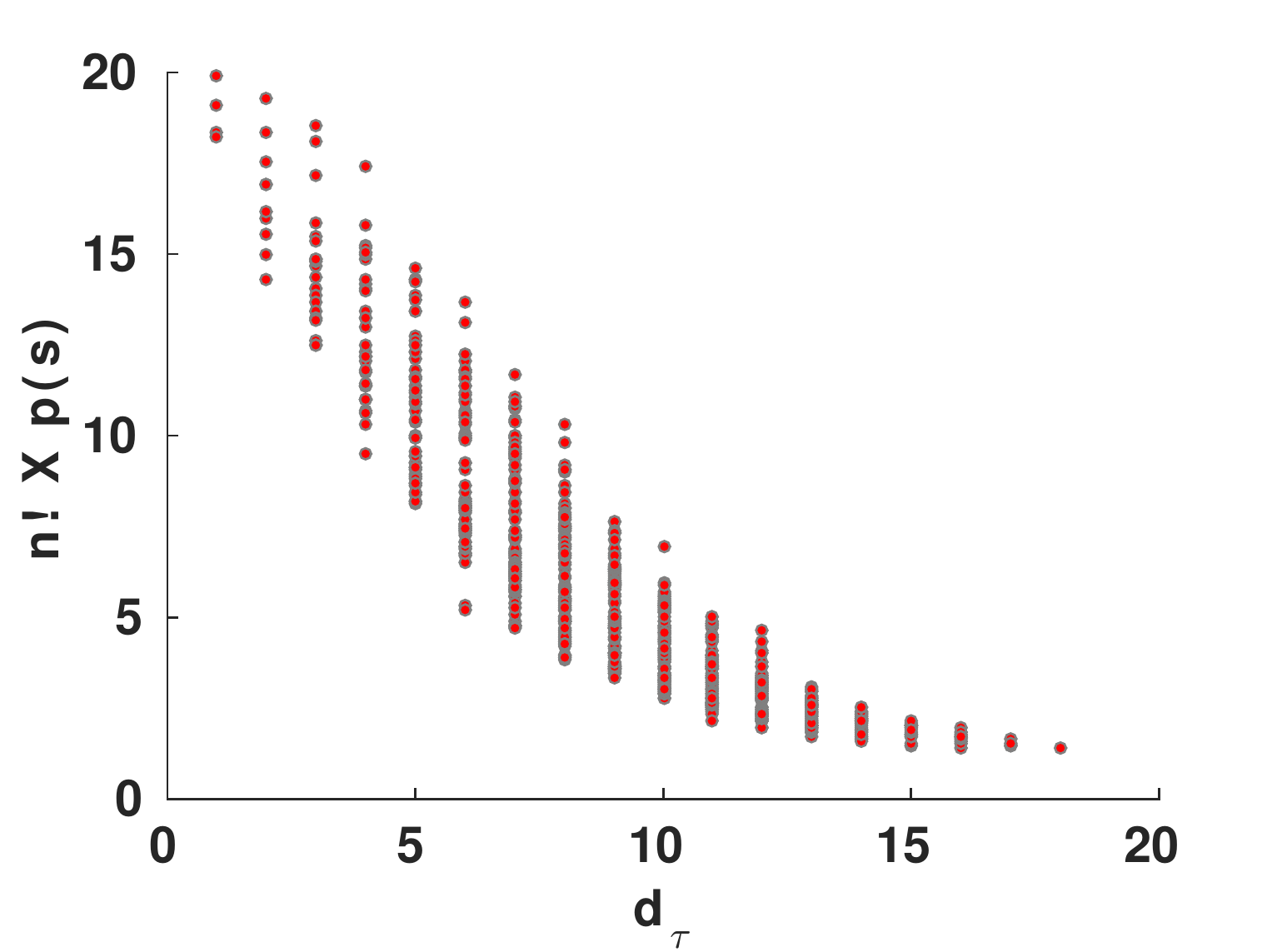}
\end{tabular}
\end{center}
% \figurebox{5cm}{}{}[compatranksAsym]
% \figurebox{5cm}{}{}[compatranksJeff]
\caption{Rescaled predictive probabilities, $n! \, p(s)$, of all compatible rankings 
$s\in \mathcal{C}(s^*,{\cal M}^*)$, where $n=7$, ${\cal M}^*=\{2,4,5\}$, and
$s^*=(2,1,3)$, using the FGM copula family. The values are plotted against the 
Kendall distance of the rankings to the modal rank. Left: asymmetrical Beta prior $\theta=2T-1$ with $T\sim \text{Beta}(1/10,2)$. Right: Jeffreys' prior $\pi_J$.}
\label{fig:preddist}
\end{figure} 

In contrast to the marginal distribution, the posterior predictive distribution arising from Jeffrey's prior is no longer symmetric around the midrange distance $\binom{n}{2}/2$: compare the right-hand panels of Figures~\ref{fig:mardist} and~\ref{fig:preddist}. Recall that the symmetry property of the marginal distribution is due to a combination of equations~\eqref{eq:symeq} and~\eqref{eq:symdist}. For the predictive distribution, this explanation breaks down, because $s \in {\cal C}(s^*, {\cal M}^*)$ implies $a \circ s \not\in {\cal C}(s^*, {\cal M}^*)$. Indeed, a 
compatible ranking $s \in \mathcal{C}(s^*,{\cal M}^*)$ must satisfy 
\[
  s(i_{\sigma(1)}^*) < s(i_{\sigma(2)}^*) < \cdots <s(i_{\sigma(m)}^*),
  \qquad \sigma=(s^*)^{-1},
\]
but for $s' = a \circ s = (n+1-s(1),\ldots,n+1-s(n))$, we have $s'(i_{\sigma(1)}^*) > s'(i_{\sigma(2)}^*) > \cdots > s'(i_{\sigma(m)}^*)$, and so $a \circ s \notin \mathcal{C}(s^*,{\cal M}^*)$. In passing, note that, in contrast to the ranking $a \circ s$, the ranking $s\circ a$ may or may not belong to the compatible rankings. Take for instance $n = 3$, ${\cal M}^* = \{1, 2\}$, and $s^* = (1, 2)$. On the one hand, we have $e = (1, 2, 3) \in \mathcal{C}(s^*, {\cal M}^*)$ but $a = (3, 2, 1) \notin \mathcal{C}(s^*, {\cal M}^*)$. On the other hand, if $s = (2, 3, 1)$, we have both $s \in \mathcal{C}(s^*, {\cal M}^*)$ and $s \circ a = (1, 3, 2) \in \mathcal{C}(s^*, {\cal M}^*)$.

% ==================
\section{Algorithms}
\label{sec:algo}

% The mode of the posterior predictive distribution is the ranking used in order to make a recommendation to the individual. The simulated annealing algorithm proposed in Section~\ref{S:Algo1} below gives a way to approximate this mode. It is based on an algorithm to draw random compatible rankings in Section~\ref{S:Algo0}. However, it can be interesting for a recommendation system developper to explore more than a single ranking in order to make the recommendation to the individual. One reason may be to prevent local maxima situations inherent to simulated annealing algorithms. We therefore propose a second, more involved algorithm to obtain the few most likely rankings according to the predictive distribution. In fact, the entire predictive distribution is approximated in Section~\ref{S:Algo2} below.

% --------------------------------------
\subsection{Drawing compatible rankings}
\label{S:Algo0}

The mode of the posterior predictive distribution is the ranking used in order to make a recommendation to the individual. The simulated annealing algorithm proposed in Section~\ref{S:Algo1} below gives a way to approximate this mode. It is based on an algorithm to draw random compatible rankings. In practice, the cardinality, $n!/m!$, of $\mathcal{C}(s^*,{\cal M}^*)$ can be enormous, and a complete listing of all compatible rankings is elusive. One way to draw samples from the uniform distribution over ${\cal C}(s^*, {\cal M}^*)$ is to draw a permutation $\tilde{s}$ randomly from $\Sgrp_n$ and then turn it to a compatible ranking $s \in {\cal C}(s^*, {\cal M}^*)$ by rearranging $(\tilde{s}(i_1^*), \ldots, \tilde{s}(i_m^*))$ in such a way that $\rank(s(i_1^*), \ldots, s(i_m^*)) = s^*$. This algorithm could be used for constructing Markov chain Monte Carlo algorithms with independent proposals. Here, we are interested in random walk type proposals, and so we construct an ergodic Markov chain on ${\cal C}(s^*, {\cal M}^*)$ which happens to have a uniform stationary distribution. It will be used as a proposal distribution in Algorithms~\ref{algo:modeapprox} and~\ref{algo:predapprox} below.

\begin{algo}
\label{algo:compat}
Let $s \in {\cal C}(s^*, {\cal M}^*)$ be the current state of the chain. The next state $s' \in {\cal C}(s^*, {\cal M}^*)$ is obtained by selecting at random one 
move between moves $\Move_{1}$ and $\Move_{2}$ with equal probability.
\begin{compactenum}[${\Move}_1$ --]
\item \emph{Swap move.} Draw a pair $\{i,j\}$ where $i,j\in\mathcal{M}^c$, $i<j$ (with probability 
$1/\binom{n-m}{2}$), and take $s'$ such that 
$s'(i)=s(j)$, $s'(j)=s(i)$, and $s'(t)=s(t),$ for every $t\in\{1,\ldots,n\}\setminus\{i,j\}.$
\item \emph{Swap and rearrange move.} Draw $\ell\in\{1,\ldots,m\}$ and 
$j\in\mathcal{M}^c$ (with probability 
$1/[m(n-m)]$) and then take $s'$ such that $s'(j)=s(i_{\ell}^*)$, $s'(t)=s(t)$, 
for every $t\in 
\mathcal{M}^c\setminus\{j\}$, and such that $\{s'(i_k^*)\colon 
k=1,\ldots,m\}=\{s(i_k^*)\colon k=1,\ldots,m, 
k\not=\ell\}\cup\{s(j)\}$, with 
\[s'(i_{\sigma(1)}^*)<s'(i_{\sigma(2)}^*)<\cdots <s'(i_{\sigma(m)}^*), \quad 
\sigma=(s^*)^{-1}.\]
\end{compactenum}
\end{algo}

\begin{lemma}
\label{lem:compat}
If $1<m<n$, then Algorithm~\ref{algo:compat} generates an irreducible and aperiodic Markov chain with uniform stationary distribution on ${\cal C}(s^*, {\cal M}^*)$.
\end{lemma}

% -------------------------------------------------------------------
\subsection{Finding the modal ranking of the predictive distribution}
\label{S:Algo1}

A simple way to compute the predictive probabilities $p(s)$ in~\eqref{eq:preddist} is by standard Monte-Carlo approximations, exploiting Theorem~\ref{thm:RL} for the rank likelihood. Combining this with draws from the compatible rankings in ${\cal C}(s^*, {\cal M}^*)$ according to Algorithm~\ref{algo:compat}, we propose the following simulated annealing algorithm for approximating~\eqref{eq:shat}.

\begin{algo}
\label{algo:modeapprox}
At each iteration $t$, let $S_t \in {\cal C}(s^*, {\cal M}^*)$ be the current 
state of the rankings. 
\begin{compactenum}[$\text{SA}_1$.]
\item Draw $S'$ according to Algorithm~\ref{algo:compat}.
\item For $k=1,\ldots, K$, repeat move $\text{MC}_1$ 
and when completed, do move $\text{MC}_2$.
\begin{compactenum}[$\text{MC}_1$.]
\item Draw 
$(W_{\ell,1}',\ldots,W_{\ell,n+1}')=(W_{\ell,1}',\ldots,W_{\ell,n}',1-\sum_{j=1}
^nW_{\ell,j}')$, for $\ell = 1, 2$, from $\operatorname{Dirichlet}(1,\ldots,1)$, and independently draw $\theta'\sim \pi$, and evaluate
\begin{equation*}
 \label{eq:pk}
p_k(S')=\frac{1}{n!}\prod_{i=1}^n c_{\theta'}\Big(\sum_{j=1}^i
W_{1,j}',\sum_{j=1}^{S'(i)} W_{2,j}'\Big).
\end{equation*}
\item Compute $\hat{p}(S')=\frac{1}{K}\sum_{k=1}^Kp_{k}(S').$
\end{compactenum} 
\item Set $S_{t+1}=S'$ with probability
\[
 1\wedge \exp\left\{\frac{\hat{p}(S')-\hat{p}(S_t)}{T_t}\right\},
\]
where $T_t=1/\log t$ is the temperature. Otherwise, $S_{t+1}=S_t$.
\end{compactenum}
\end{algo}

As an illustration, we consider the toy example of 
Figure~\ref{fig:preddist}, with incomplete permutation $(-,2,-,1,3,-,-)$, 
the \FGM\ model and the Jeffreys' prior on the copula parameter. In Figure~\ref{fig:SA}, the excursion eventually oscillates between the two modes $s=(1,4,2,3,5,6,7)$ and 
$s^{-1}=(1,3,4,2,5,6,7)$, found at the beginning of this section, and $d_{\tau}(s,s^{-1})=2$.

\begin{figure}
\begin{center}
\includegraphics[width=0.4\textwidth]{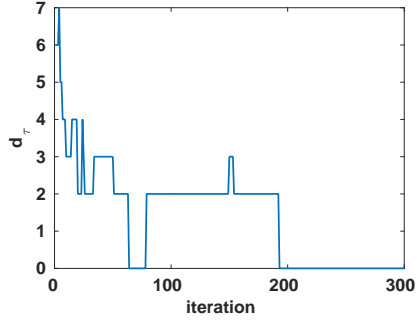}
\end{center}
% \figurebox{5cm}{}{}[SAI]
\caption{Excursion $S_1, S_2, \ldots$ from the simulated annealing Alorithm~\ref{algo:modeapprox} excursion. Here, as in Figure~\ref{fig:preddist}, the incomplete ranking is $n = 7$, ${\cal M}^* = \{2, 4, 5\}$ and $s^* = (2, 1, 3)$, that is, $(-,2,-,1,3,-,-)$. The illustration shows the Kendall distance, $d_\tau$, of the visited permutations from the mode plotted against the iteration number. The \FGM\ copula parameter $\theta$ is distributed according to Jeffreys' prior $\pi_J$.}
\label{fig:SA}
\end{figure}

% ----------------------------------------------------
\subsection{Approximating the predictive distribution}
\label{S:Algo2}

It can be interesting for a recommendation system developper to explore more than a single ranking in order to make the recommendation to the individual. One reason may be to prevent local maxima situations inherent to simulated annealing algorithms. We therefore propose a second, more involved algorithm to obtain the few most likely rankings according to the predictive distribution %. In fact, the entire predictive distribution is approximated in Section~\ref{S:Algo2} below.
%
%We propose a more involved algorithm to approximate the entire predictive distribution 
$p(s)$, $s\in {\cal C}(s^*, {\cal M}^*)$. The idea is to construct a Markov chain with limiting (stationary) distribution equal to the conditional distribution of $(S, W_1, W_2, \theta)$ given $\{S \in \mathcal{C}(s^*, {\cal M}^*)\}$. In line with Theorem~\ref{thm:RL}, the random vector $(S, W_1, W_2, \theta)$ has joint density given by
\begin{equation}
\label{eq:joint}
  f(s,w_1,w_2,\theta)
  = n! \, \prod_{i=1}^n c_{\theta}
  \Bigl( \sum_{j=1}^i w_{1,j}, \sum_{j=1}^{s(i)} w_{2,j}\Bigr) \, \pi(\theta), 
  \qquad s \in \Sgrp_n, (w_1,w_2)\in \Delta^2, \theta \in \Theta,
\end{equation}
with respect to $\nu \times \lambda_1 \times \lambda_2$, where $\nu$ is the counting measure, and $\lambda_1$ and $\lambda_2$ are the Lebesgue measures on $\Delta^2$ and $\Theta$, respectively, with $\Delta = \{ w \in (0,1)^n : w_1 + \cdots + w_n < 1 \}$.

Essentially, $S_1,S_2,\ldots$ travels through the space ${\cal C}(s^*, {\cal M}^*)$, with limiting relative frequencies of occupancy converging to the predictive probabilities $p(s)$ in~\eqref{eq:preddist}. The move $\Move_2$ below concerns the proposal of the variables $(W_1,W_2)$ given $S$ and $\theta$. We have looked at two ways to do so, resulting in two variations of the algorithm: the proposal is drawn either independently of the current value of $(W_1,W_2)$ or from the instrumental density in equation~\eqref{eq:instr} below. The two variations will be called (\textsc{mhi}) and (\textsc{mhrw}) in Move $\Move_{2}$ in Algorithm~\ref{algo:predapprox}. Move $\Move_{3}$ of the algorithm requires a draw of a $\theta' \in \Theta$ according to some instrumental density $q( \theta' \mid \theta)$, given the previous value $\theta \in \Theta$; for the \FGM\ family, this is the density of the uniform distribution on the interval $[\theta - \eps, \theta + \eps] \cap [-1, 1]$, for some tuning parameter $\eps > 0$.

\begin{algo}
\label{algo:predapprox}
Let $(S,W_1,W_2,\theta)$ be the initial state of the chain, with $S \in 
{\cal C}(s^*, {\cal M}^*)$, $(W_1,W_2)\in 
\Delta^2$, and $\theta \in \Theta$. Let $f$ be the density given 
in~\eqref{eq:joint}. 
\begin{compactenum}
\item
At each iteration $t=1,\ldots,N$, select a 
move at random (equiprobably) between moves $\Move_{1}$, $\Move_{2}$, and 
$\Move_{3}$.
\begin{compactenum}[${\Move}_1$ --]
\item Draw $S'$ according to Algorithm~\ref{algo:compat} and replace $S$ by 
$S'$ with probability
\[
 1\wedge \frac{f(S',W_1,W_2,\theta)}{f(S,W_1,W_2,\theta)}.
\]
\item 
\begin{compactitem}
\item[(\textsc{mhi})] 
Draw $W_{\ell}' = (W_{\ell,1}',\ldots,W_{\ell,n+1}') \sim \operatorname{Dirichlet}(1,\ldots,1)$, for $\ell=1,2$, and
and replace $(W_1,W_2)$ by $(W_1',W_2')$ with probability
\begin{equation}
\label{eq:accept}
  1 \wedge \frac{f(S,W_1',W_2',\theta)}{f(S,W_1,W_2,\theta)}.
\end{equation}
\item[(\textsc{mhrw})] 
Choose $\ell \in \{1, 2\}$ at random
and draw $W_{\ell}' = (W_{\ell,1}',\ldots,W_{\ell,n}')$ according to the distribution with density given by~\eqref{eq:instr} with $w_0=W_{\ell}$. Put 
$W_j'=(W_{j,1}',\ldots,W_{j,n}')=(W_{j,1},\ldots,W_{j,n})$ for $j\in 
\{1,2\}\setminus \{\ell\}$, and replace 
$(W_1,W_2)$ by $(W_1',W_2')$ with probability
\[
  1 \wedge 
  \frac%
    {f(S,W_1',W_2',\theta) \, q_{W_{\ell}} (W_{\ell}  \mid W_{\ell}')}%
    {f(S,W_1, W_2, \theta) \, q_{W_{\ell}'}(W_{\ell}' \mid W_{\ell} )}.
\]
\end{compactitem}
\item Draw $\theta'$ according to some density $q(\theta' \mid \theta)$; we use $\theta'\mid \theta \sim U(-1 \vee \{\theta-\eps\}, 1 \wedge \{\theta+\eps\})$. Replace $\theta$ by $\theta'$ with probability
\[
  1 \wedge \frac{f(S,W_1,W_2,\theta') \, q(\theta \mid \theta')}{f(S,W_1,W_2,\theta) \, q(\theta' \mid \theta)}.
\]
\end{compactenum}
The current state of the chain then becomes the (possibly unchanged) state, denoted for simplicity by $(S,W_1,W_2,\theta)$; set $S_t=S$.
\item 
Let $\mathcal{O}(s^*, {\cal M}^*) = \{ S_1, \ldots, S_N \}$ be the set of all the (distinct) values taken by $S_1,\ldots,S_N$. For each $s\in \mathcal{O}(s^*,{\cal M}^*)$, compute the relative frequency of $s$: 
\[
  \hat{p}(s) = \frac{1}{N} \sum_{t=1}^N \indic(S_t = s).
\]
\end{compactenum} 
\end{algo}

\begin{lemma}
\label{thm:instrumental}
Let $w_0=(w_{01},\ldots,w_{0n})\in \Delta$. If $W' \sim \operatorname{Dirichlet}(1,\ldots,1)$ and if $\Lambda$ is an independent random variable on $(0, 1)$ with density $g$, then $W=(1-\Lambda)w_0+\Lambda W'$ has density
\begin{equation}
\label{eq:instr}
  q_W(w\mid w_0)
  =
  n! \, \indic_{\Delta}(w) 
  \int_{1-\delta(w;w_0)}^1 \lambda^{-n} \, g(\lambda) \, \diff\lambda,
\end{equation}
where $\delta(w;w_0)=\min\{\frac{w_i}{w_{0i}} : 1 \leq i \leq n\} \wedge \{ (1 - \sum_{i=1}^n w_i) / (1 - \sum_{i=1}^n w_{0i}) \}$.
\end{lemma}

% \paragraph{Comparison of (MHI) and (MHRW) in Algorithm~\ref{algo:predapprox}.}
Let us compare the two variations, (\textsc{mhi}) and (\textsc{mhrw}), of Algorithm~\ref{algo:predapprox} and use the \FGM\ family as a validating benchmark to assess their accuracy in approximating the entire predictive distribution. We look at the
total variation distance from the true predictive distribution, $p$, to the approximation $\hat{p}_t$,
as a function of the iteration $t$. Here, the total variation distance is defined as
\[
 \operatorname{TV}(\hat{p}_t,p)=\frac{1}{2}\sum_{s\in 
{\cal C}(s^*, {\cal M}^*)} \lvert \hat{p}_t(s)-p(s) \rvert.
\]
We have again taken the same situation as Figure~\ref{fig:preddist} with Jeffreys' prior and we have considered the six incomplete rankings with $n=7$ and ${\cal 
M}^*=\{2,4,5\}$, one such ranking for every $s^*\in \Sgrp_3$. 
% Figure~\ref{fig:TV} below gives the results.
Figure~\ref{fig:TV} shows the results for $s^* = (1,2,3)$ and $s^* = (3,2,1)$. The results for the other four permutations in $\Sgrp_2$ are similar.

\begin{figure}
\begin{center}
\begin{tabular}{cc}
{\includegraphics[width=0.4\textwidth]{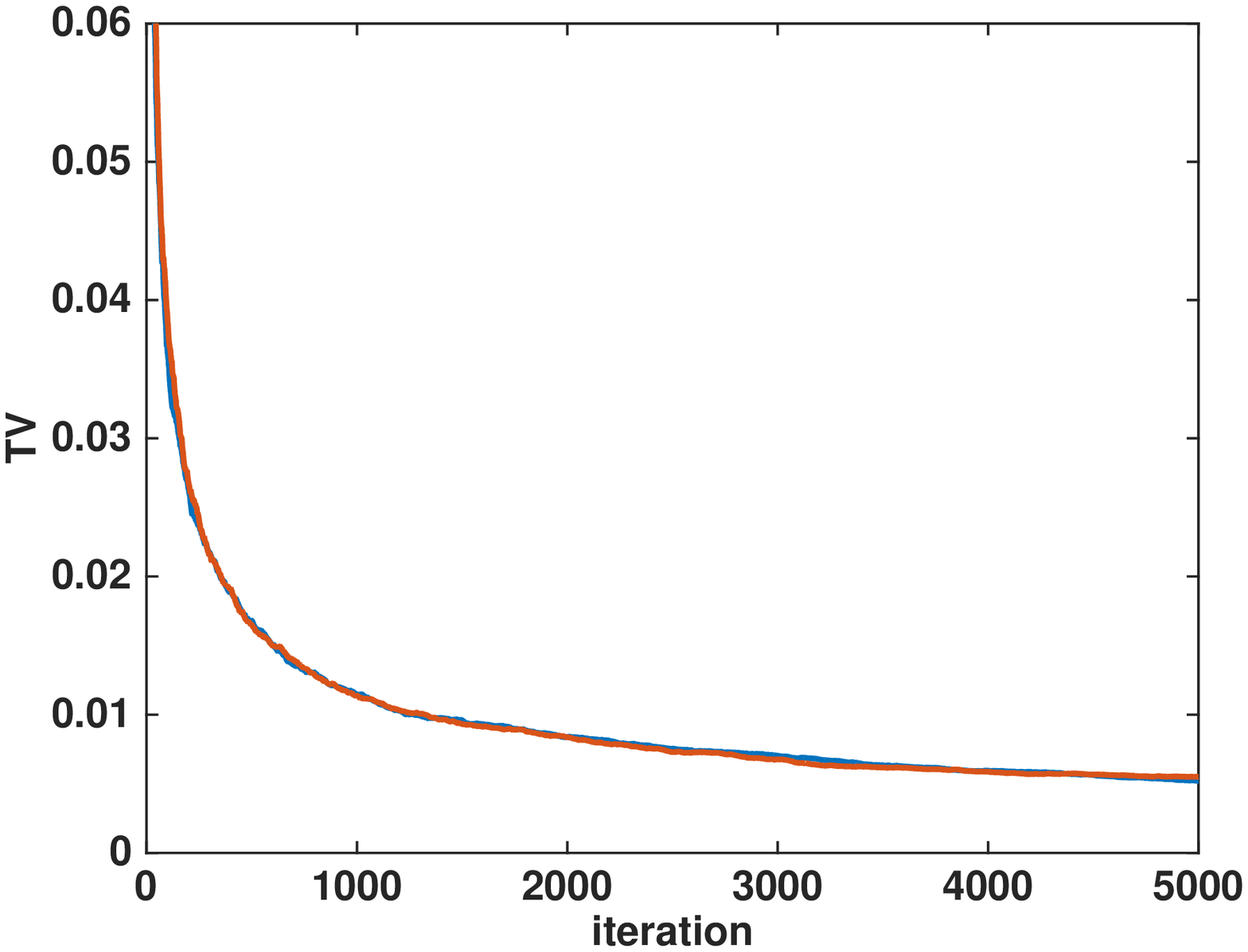}}&
{\includegraphics[width=0.4\textwidth]{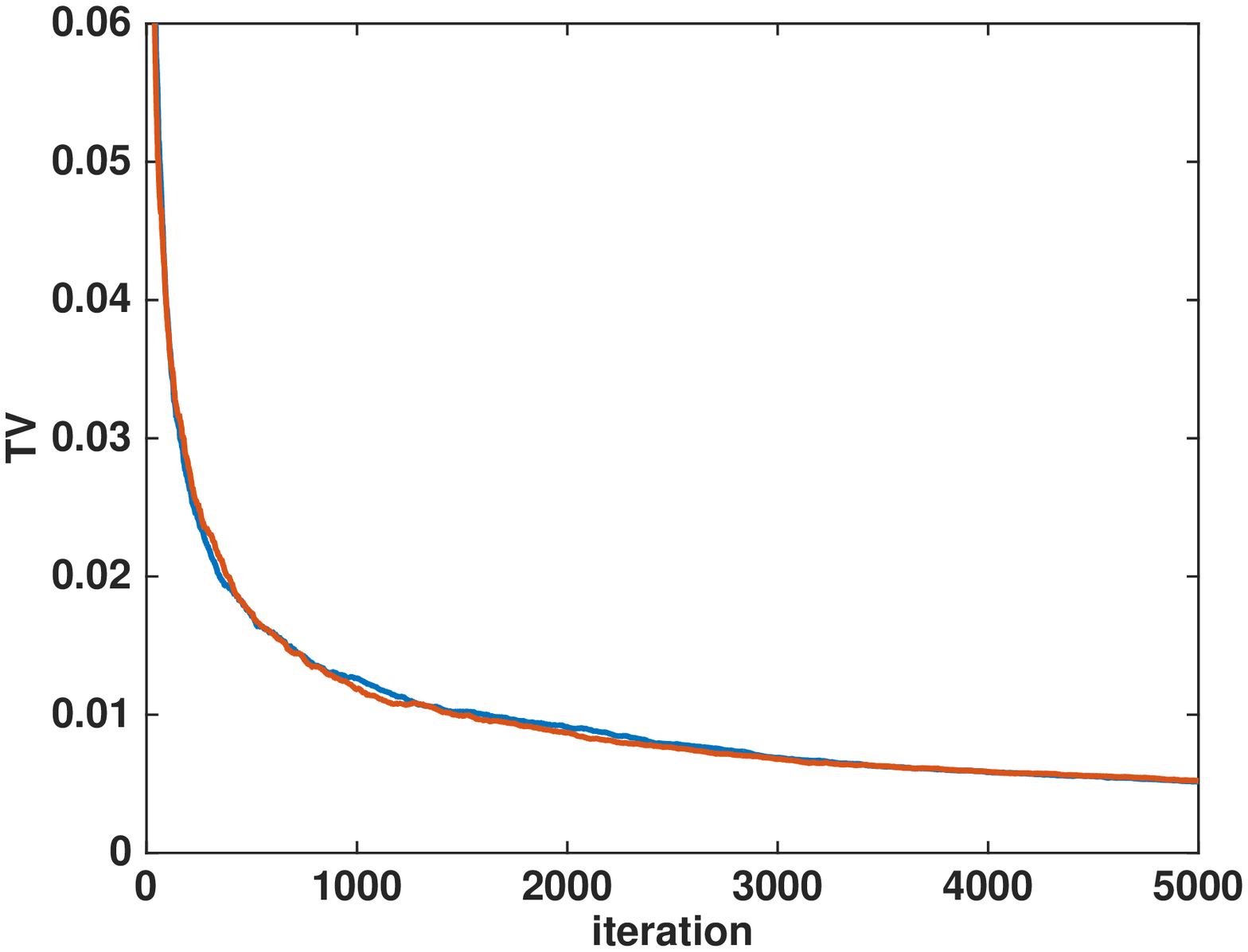}}\\
(a) $s^*=(1,2,3)$&
(f) $s^*=(3,2,1)$
\end{tabular}
\end{center}
\caption{Total variation distance of the true predicted distribution to the approximation using  Algorithm~\ref{algo:predapprox} (\textsc{mhi}) in blue and (\textsc{mhrw}) in orange, using Jeffreys' prior on the \FGM\ copula parameter $\theta$. The incomplete rankings are given by $n=7$, ${\cal M}^*=\{2,4,5\}$ and $s^* = (1, 2, 3)$ (left) and $s^* = (3,2, 1)$ (right). The values are plotted against the iteration number (in thousands).%
%The illustrations show the different results obtained with Jeffreys' prior on the \FGM\ copula parameter $\theta$.
}
\label{fig:TV}
\end{figure}

% \begin{figure}
% \begin{center}
% \begin{tabular}{cc}
% {\includegraphics[width=0.4\textwidth]{art/TV123}}&
% {\includegraphics[width=0.4\textwidth]{art/TV132}}\\
% (a) $s^*=(1,2,3)$&
% (b) $s^*=(1,3,2)$\\
% {\includegraphics[width=0.4\textwidth]{art/TV213}}&
% {\includegraphics[width=0.4\textwidth]{art/TV231}}\\
% (c) $s^*=(2,1,3)$&
% (d) $s^*=(2,3,1)$\\
% {\includegraphics[width=0.4\textwidth]{art/TV312}}&
% {\includegraphics[width=0.4\textwidth]{art/TV321}}\\
% (e) $s^*=(3,1,2)$&
% (f) $s^*=(3,2,1)$
% \end{tabular}
% % \subfigure[$s^*=(1,2,3)$]
% % {\includegraphics[width=5cm]{TV123}}
% % \subfigure[$s^*=(1,3,2)$]
% % {\includegraphics[width=5cm]{TV132}}
% % \subfigure[$s^*=(2,1,3)$]
% % {\includegraphics[width=5cm]{TV213}}
% % \subfigure[$s^*=(2,3,1)$]
% % {\includegraphics[width=5cm]{TV231}}
% % \subfigure[$s^*=(3,1,2)$]
% % {\includegraphics[width=5cm]{TV312}}
% % \subfigure[$s^*=(3,2,1)$]
% % {\includegraphics[width=5cm]{TV321}}
% \end{center}
% \caption{Total variation distance of the true predicted distribution to the approximation using the 
% two algorithms:  Algorithm~\ref{algo:predapprox} ({\color{mblue}MHI} and {\color{morange}MHRW}). The incomplete rankings are given by $n=7$, ${\cal 
% M}^*=\{2,4,5\}$ and all six values of $s^*\in S_3$. The 
% values are plotted against the iteration number (in thousands). The 
% illustrations show the different results obtained 
% with Jeffreys' prior on the FGM copula parameter $\theta$.}
% \label{fig:TV}
% \end{figure} 

Although roughly noticeable, there is and should be a certain symmetry in 
the results between Figure~\ref{fig:TV} (a) 
and (b).
% and (f), between (c) and (d), and between (b) and (e). 
This is explained by the equality of the events 
\[
  \{ S \in {\cal C}(a\circ s^*,{\cal M}^* )\}
  =
  \{ a \circ S \in {\cal C}(s^*, {\cal M}^*) \},
\] 
where, by a little abuse of notation, $a$ stands for the anti-identity in $\Sgrp_m$ on the left-hand side and for the anti-identity in $\Sgrp_n$ on the right-hand side. It follows that for the \FGM\ copula family and a symmetrical prior like Jeffreys' prior, since $\prob ( a \circ S = s ) = \prob ( S = s )$ for all $s \in \Sgrp_n$, we obtain
\begin{align*}
  \prob \{ S = s \mid S \in {\cal C}(a\circ s^*,{\cal M^*}) \}
  &=
  \prob \{ S = s \mid a \circ S \in {\cal C}(s^*, {\cal M}^*) \} \\
  &= 
  \prob \{ S = a \circ s \mid S \in {\cal C}(s^*, {\cal M}^*) \}
  .
\end{align*}
% for every $s \in \Sgrp_n$.

\section{Comparisons with other recommender systems}
\label{S:Data}

% \noindent \textbf{The competitors}. 
We compare our method in Algorithm~\ref{algo:modeapprox} with those available from the \textit{Personalized Recommendation Algorithms} (PREA) java software, see~\cite*{LSL12}, which contains some state-of-the-art techniques. Our competitors are the algorithms \textit{SlopeOne}, introduced by~\cite{LM05}, \textit{Non-negative Matrix Factorization} (NMF), see for instance~\cite{LS99}, \textit{Probabilistic Matrix Factorization} (PMF), see~\cite{SM07}, \textit{Bayesian Probabilistic Matrix Factorization} (BPMF), see~\cite{SM08}, \textit{Regularized Singular Value Decomposition} (RegSVD), see~\cite{Paterek07}, and \textit{Fast Non-negative Principal Component Analysis} (NPCA), see~\cite*{Yu09}. We refer to our method as \textit{Bayesian Bivariate Ranks} (BBR).

% \noindent \textbf{The data}. 
We consider the well known \textit{MovieLens 100k} (https://grouplens.org/datasets/movielens/100k/) dataset, comprising 1664 users, 943 movies, with a total of 99392 ratings, from which we select a subset of users and movies with many ratings to act as all the data. More precisely, we consider a matrix of 100 users and 35 movies, with 2687 available ratings, from one to 5 stars. In the selected matrix, the movies are ranked in order from the highest rated movie (by averaging the user ratings of each movie in the entire \textit{MovieLens 100k} dataset) to the lowest rated one. Note that this overall ranking is important only for our method and acts as the expert opinion. In view of the notation of Section~\ref{S:RL}, this ordering implies that $r_x=e$ and $r_y=s$ for every user, although the movies ranked for one user may differ from that of another user. 

The user ratings have ties. While this does not cause problems for the overall ordering of the movies using the entire \textit{MovieLens 100k} dataset, it does require a choice for obtaining an individual user's permutation $r_y=s$. We break the ties using the expert opinion, in that if a user has given identical ratings to two or movies, the overall ordering determines their mutual ranks.

% \noindent \textbf{The experiment}. 
From the ratings data, we keep (at random) a certain proportion $p$ of the data, and this we do for each user. We have considered the proportions 5\%, 10\%, 15\%, 20\%, 25\%, 50\% and 75\%, see the graphs below. For a given proportion $p$ of ratings kept, the matrix thus obtained becomes the data which we shall use to predict the rankings of the users. This is done, in turn, for each of the 100 users. Apart from the prior specification, see below, our methodology uses only the current user's retained rankings, obtained as discussed above. In contrast, the other methodologies use the entire matrix of retained ratings (not rankings), for each user. We then repeat all of this 30 times. There are many metrics used to evaluate recommender systems, see \cite{GS09} and~\cite{LSL12}, among which the Kendall distance $d_\tau$ in~\eqref{eq:dtau}. Since we are interested on the predicted rankings and not ratings, we shall focus on this distance for evaluating the methodologies. Note that the methodologies considered here do not give ties very often, especially when $p$ is large, and so the predictions, either of ratings (the competitors) or rankings (us) are all (or can be transformed into) genuine permutations. In the few encountered events where there were ties in the predicted ratings (for small $p$), again we broke the ties using the expert opinion, to obtain permutations.

% \noindent \textbf{Details for our Bivariate Bayesian Rank (BBR) methodology}. 
We have used the Gaussian copula family with parameter $-1<\rho<1$ as model for the dependence between the expert and user ratings. For the prior on $\rho$, we consider the one-to-one relation $\rho= \sin(\tau \pi/2)$, where $\tau$ is Kendall's $\tau$. We have obtained the empirical distribution of $\tau$ using the complete MovieLens data set accross all users and have applied the above transformation to obtain an estimate of the distribution of $\rho$. We finally fit the density of $2T-1$, where $T\sim \operatorname{Beta}(\alpha,\beta)$, to this estimate which gave the approximate values $\alpha=6$ and $\beta=2$. This is the prior for $\rho$.
% Note: we could have made this prior depend only on the particular user of interest (based only on his available ratings) or on the entire matrix of available ratings\ldots

% \noindent \textbf{The results}.
For a fixed proportion $p$ of available data in the matrix, and for each user $u$, let $s_u$ be his or her (true) movie rankings and let $\hat{s}_{u,k,p}$ be the prediction based on the kept data, and this for repetition $k=1,\ldots,30$. The plots in Figure~\ref{fig:meandtauk} show boxplots accross the repetitions $k$ of
\begin{equation}
\label{eq:meandtauk}
d(p,k)=\frac{1}{100}\sum_{u=1}^{100}d_{\tau}(\hat{s}_{u,k,p},s_u)
\end{equation}
for all the methods considered and for various choices of $p$, the proportion of data kept. To give a global idea of the performance of each method, Figure~\ref{fig:meandtau} shows the values of
\begin{equation}
\label{eq:meandtau}
\bar{d}(p)=\frac{1}{30}\sum_{k=1}^{30} d(p,k)
\end{equation}
as a function of $p$.

While the method that we propose seems to do better than the other methods when the information provided by the users is limited, it is also less variable than most of the other methods. In fact, the other methods that we have looked at depend only on the available users' data, whereas our method incorporates expert information. Our method could be of interest to a start-up company having little available data at first, until maybe switching too another method when the amount of data it has increases. 

\begin{figure}%[htbp]
\begin{center}
\subfigure[$p=0.05$]
{\includegraphics[width=7cm]{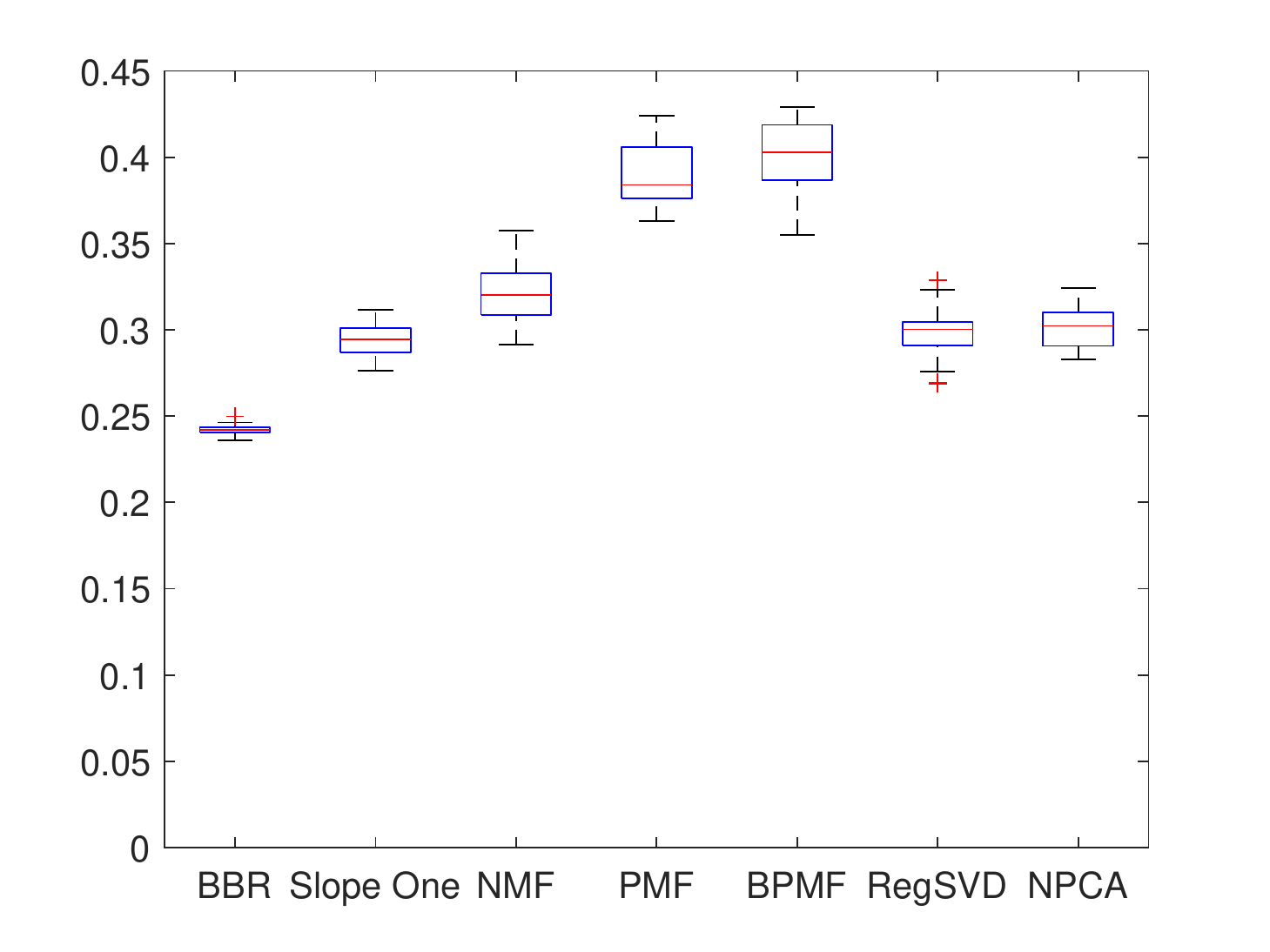}}
%\subfigure[$p=0.1$]
%{\includegraphics[width=7.5cm]{dtau10.pdf}}
\subfigure[$p=0.15$]
{\includegraphics[width=7cm]{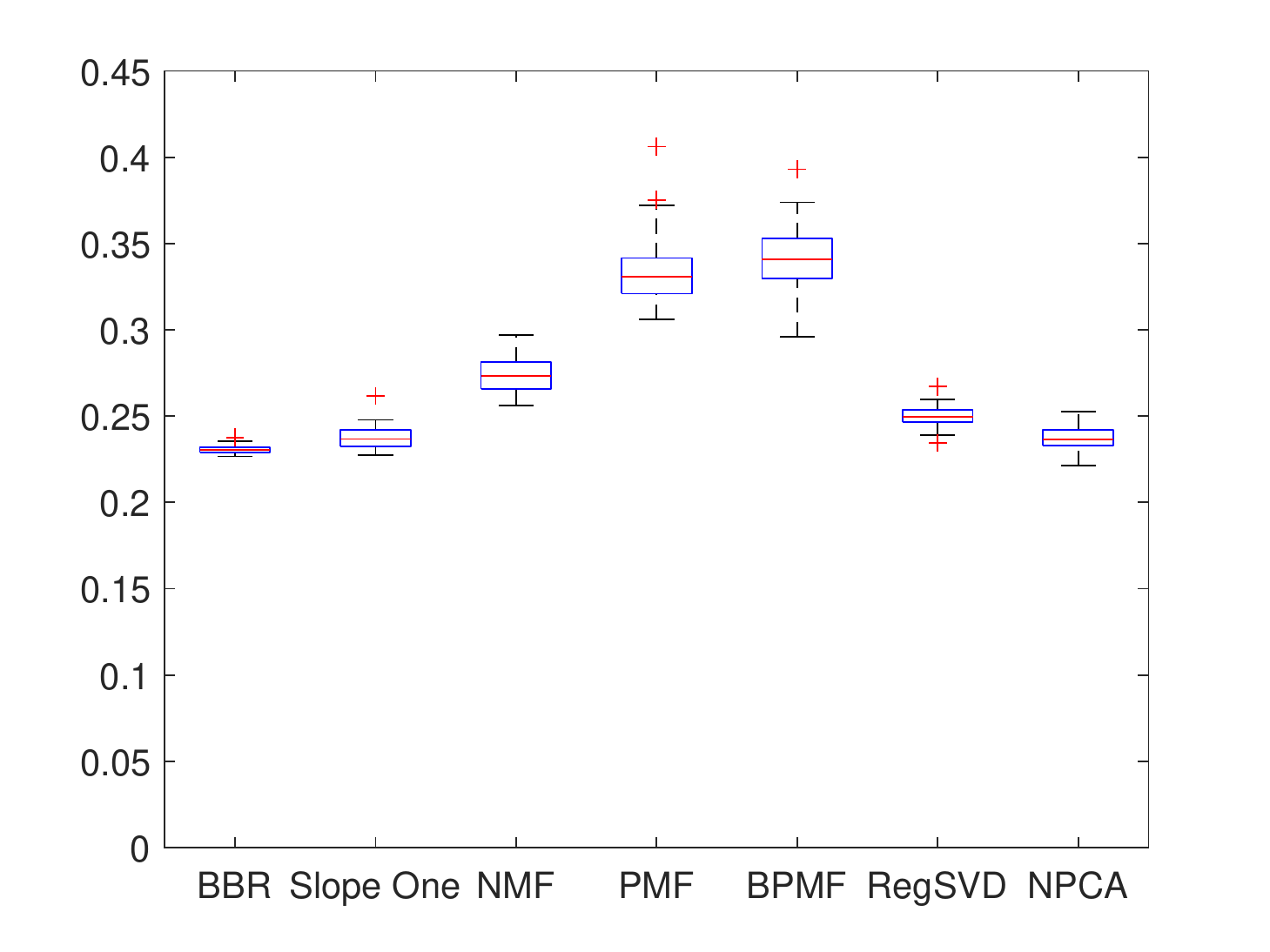}}
\subfigure[$p=0.2$]
{\includegraphics[width=7cm]{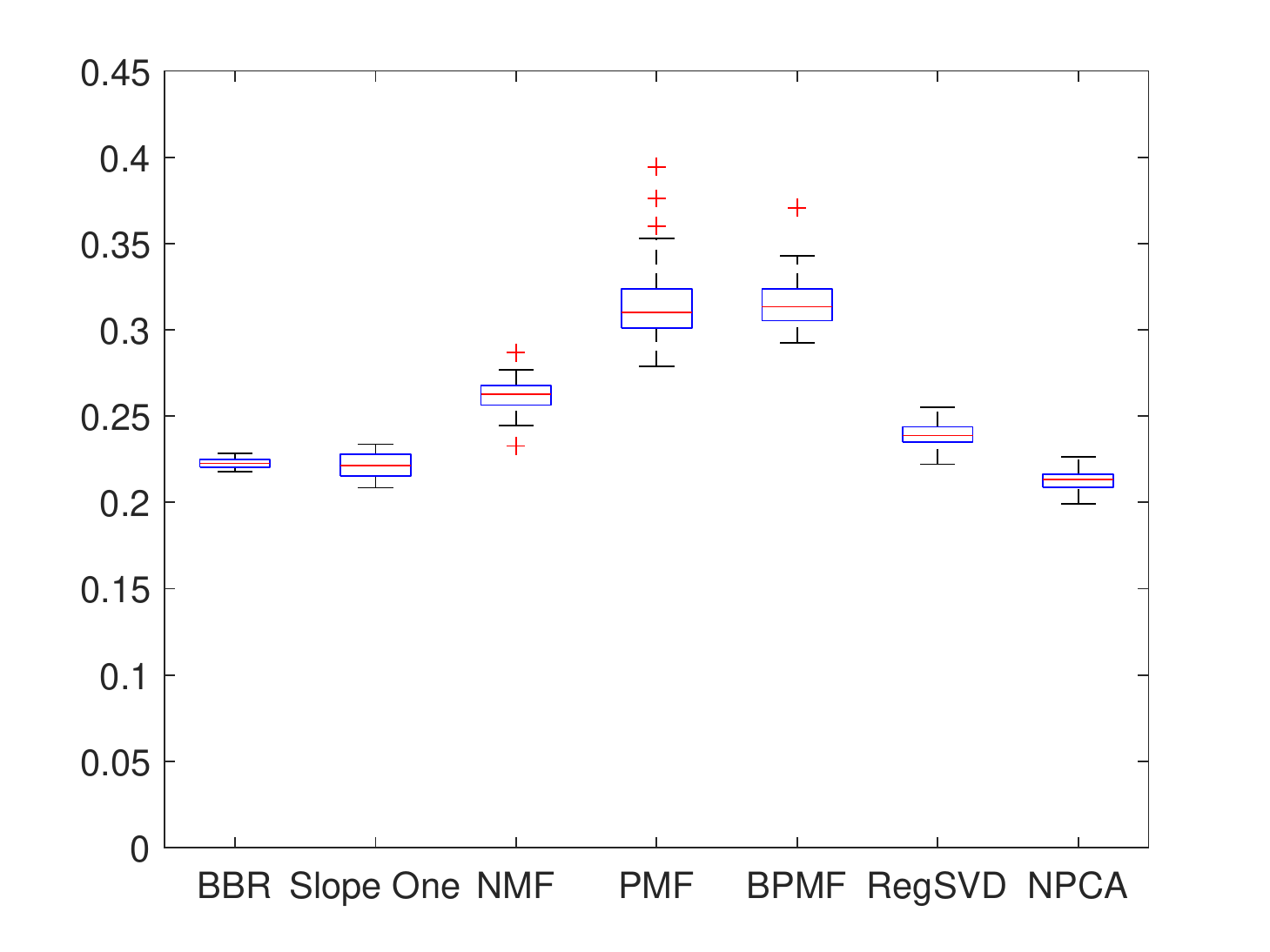}}
\subfigure[$p=0.25$]
%{\includegraphics[width=7.5cm]{dtau25.pdf}}
%\subfigure[$p=0.5$]
{\includegraphics[width=7cm]{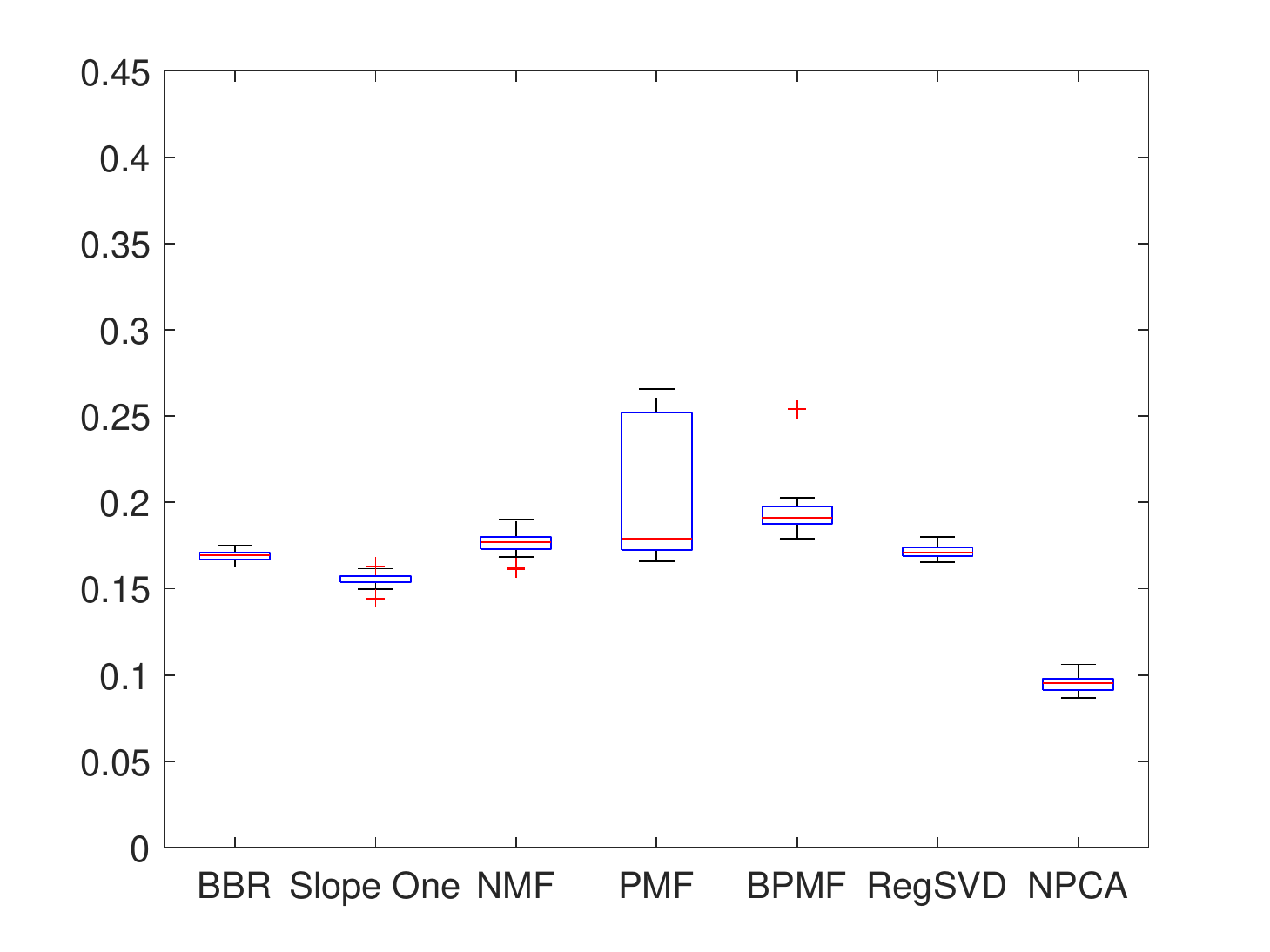}}
\subfigure[$p=0.75$]
{\includegraphics[width=7cm]{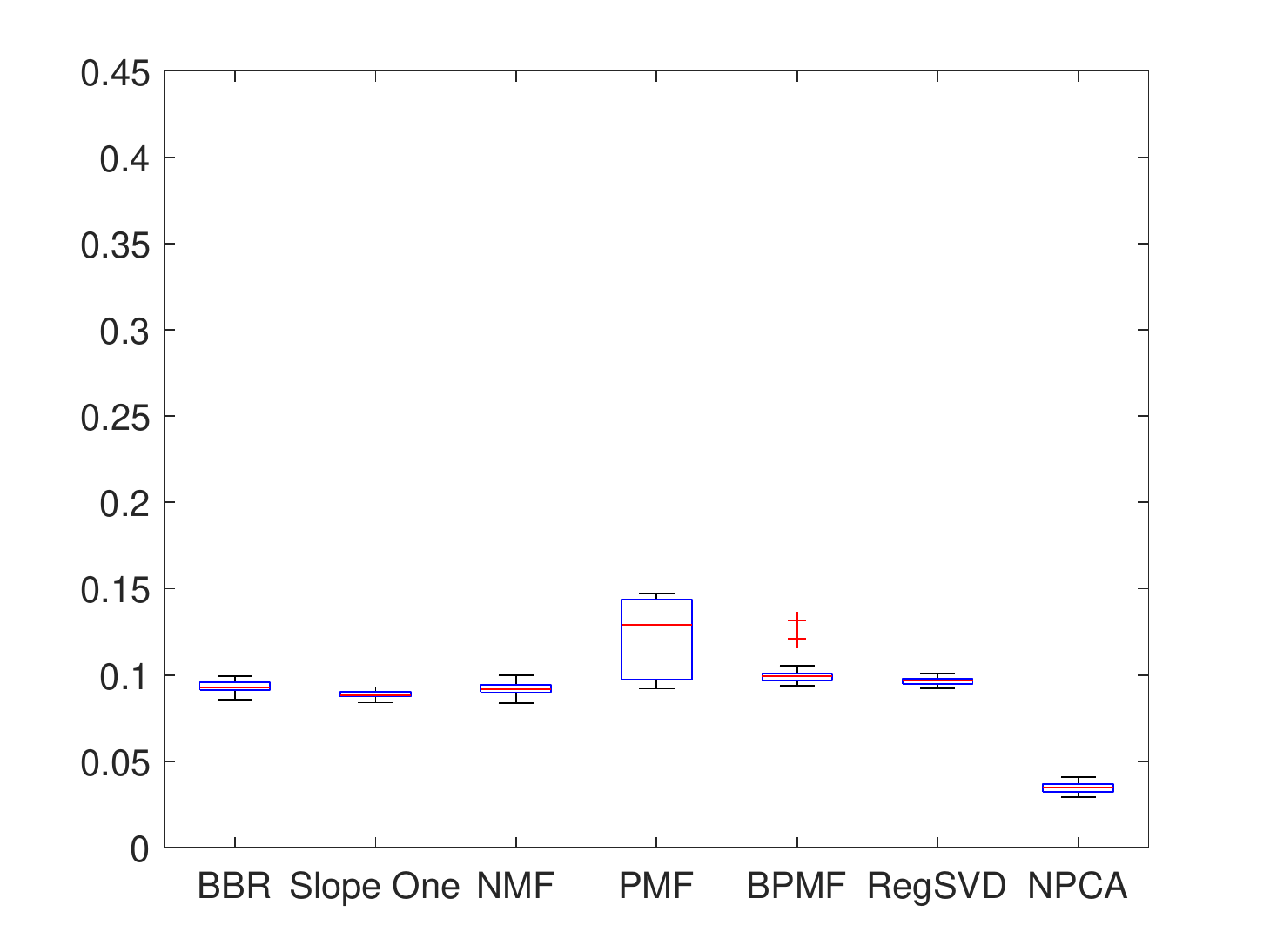}}
\end{center}
\caption{Boxplots of the $d(p,k)$ in~\eqref{eq:meandtauk} over $k=1,\dots,30$, for various values of $p$ and for various recommender systems. Our method is BBR.}
\label{fig:meandtauk}
\end{figure} 

\begin{figure}[htbp]
\centering
\includegraphics[width=10cm]{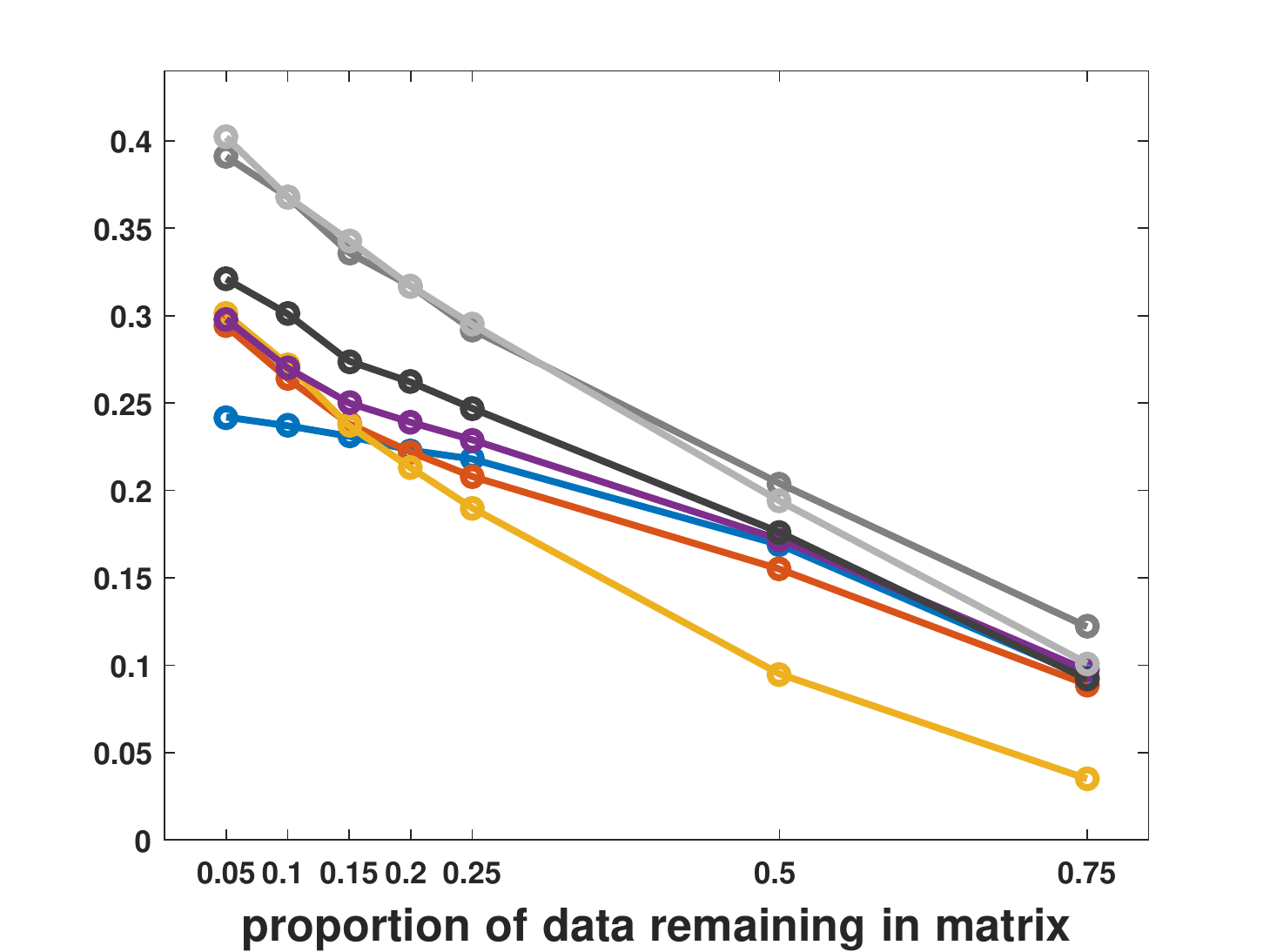}
\caption{Values of $d(p)$, given by \eqref{eq:meandtau}, for $p=5\%, 10\%, 15\%, 20\%, 25\%, 50\%$ and 75\%. Our method (BBR) is in blue, yellow is NPCA, orange is Slope One, purple is RegSVD, black is NMF, dark grey is PMF, and light grey is BPMF.}
\label{fig:meandtau}
\end{figure} 

Besides the methods mentioned, we have also tried collaborative filtering methods and local-low ranks matrix factorization methods provided by the PREA toolkit. These are useful for large-scale data but did not perform as well as the other methods considered in our experiments. 

Our method could also be used to predict the ranks of a user's top $n'$ movies more quickly. One way to do this is to run the chain on the entire set of $n$ movies with the $m$ partial (relative) rankings, therefore using all the information at hand, and then stop the chain. Consider only those $n'$ movies that have appeared most often during this first run. For those $n'$ movies, consider their relative rankings on $\Sgrp_{n'}$ obtained by the initial general (or expert) rankings. Consider also the relative $m'$ partial rankings on $\Sgrp_{m'}$ of the $m' \in \{0, 1, \ldots, m\}$ movies of the user's top $n'$ that belong to the list of the $m$ initial partial rankings. Finally, apply the algorithm to the resulting permutation, with $n$ replaced by $n'$ and $m$ by $m'$, to obtain the predicted ranking of the user's top $n'$ movies.

\section*{Acknowledgement}

J. Segers gratefully acknowledges financial support from the Projet d'Act\-ions de Re\-cher\-che Concert\'ees 
programme
% No.\ 12/17-045 
of the Communaut\'e fran\c{c}aise de Belgique and from 
% the
a
Interuniversity Attraction Pole research network grant 
% P7/06
of the Belgian federal government.

\appendix

\section{Proofs}

% \begin{proof}[of Lemma~\ref{lem:rankComp}]
% For $i \in {\cal N}$, we have
% \[
% % \begin{align*}
%   \rank( x_{\sigma(1)}, \ldots, x_{\sigma(n)} ) (i)
% %   &=
%   =
%   \sum_{j \in {\cal N}} \1( x_{\sigma(i)} \le x_{\sigma(j)} ) \\
% %   &=
%   =
%   \sum_{j \in {\cal N}} \1( x_{\sigma(i)} \le x_j )
%   =
%   \rank( x_1, \ldots, x_n )( \sigma(i) ), %\qedhere
% % \end{align*}
% \]
% as required.
% \end{proof}

\begin{proof}[Proof of Lemma~\ref{lem:S}]
For $\sigma \in \Sgrp_n$, consider the random vectors $X \circ \sigma = (X_{\sigma(1)}, \ldots, X_{\sigma(n)})$ and $Y \circ \sigma = (Y_{\sigma(1)}, \ldots, Y_{\sigma(n)})$. The joint distribution of $(X \circ \sigma, Y \circ \sigma)$ is the same as the one of $(X, Y)$. By~\eqref{eq:rankComp}, we have $R_{X \circ \sigma} = R_X \circ \sigma$ and $R_{Y \circ \sigma} = R_Y \circ \sigma$ with probability one, i.e., in the absence of ties. Setting $\sigma = \tau^{-1} \circ r_x$ with $\tau \in \Sgrp_n$, we obtain
\begin{align*}
  \prob( R_X = r_x, \, R_Y = r_y )
  &= 
  \prob( R_X \circ \tau^{-1} \circ r_x = r_x, \, R_Y \circ \tau^{-1} \circ r_x = r_y ) \\
  &=
  \prob( R_X = \tau, \, R_Y \circ \tau^{-1} = r_y \circ r_x^{-1} ) \\
  &=
  \prob\{ R_X = \tau, \, S(X, Y) = r_y \circ r_x^{-1} \}.
\end{align*}
Summing over $\tau \in \Sgrp_n$, we find that
\begin{align*}
  \prob\{ S(X, Y) = r_y \circ r_x^{-1} \}
%   &=
  =
  \sum_{\tau \in \Sgrp_n} \prob\{ R_X = \tau, \, S(X, Y) = r_y \circ r_x^{-1} \} %\\
%   &=
  =
  n! \, \prob( R_X = r_x, \, R_Y = r_y ),
\end{align*}
since the cardinality of $\Sgrp_n$ is $n!$.
\end{proof}

\begin{proof}[Proof of Theorem~\ref{thm:RL}]
As explained before the theorem, we have $R_X = R_U$ and $R_Y = R_V$ almost surely, and thus $S(X, Y) = S(U, V)$ almost surely. Let $s \in \Sgrp_n$ and let 
\[
  E_{s} = \{ (u,v) \in (0,1)^{2n} : 
    u_1 < \ldots < u_n, \,
    v_{s^{-1}(1)} < \ldots < v_{s^{-1}(n)} \}.
\]
Since $(0, 1)^{2n} \setminus \mathbb{D}_n^2$ has Lebesgue measure zero, we find, by Lemma~\ref{lem:S} with $r_x = e$,
% \begin{eqnarray}
\begin{equation}
% \nonumber
  \prob\{ S(X, Y) = s \}
%   &=&
  =
  n! \prob( R_U = e, \, R_V = s ) %\\
%   &=&
  =
  n! \int_{E_s} \prod_{i=1}^n c(u_i, v_i) \, \diff u_i \, \diff v_i.
\label{eq:RL:integral}
\end{equation}
% \end{eqnarray}
For $(u, v) \in (0, 1)^{2n} \cap \mathbb{D}_n^2$, the vector $(u_{(1)}, \ldots, u_{(n)}; v_{(s(1))}, \ldots, v_{(s(n))})$ belongs to $E_s$; here, $(z_{(1)}, \ldots, z_{(n)})$ denotes the vector of ascending order statistics of the vector $(z_1, \ldots, z_n) \in \mathbb{D}_n$. We find
\begin{equation}
% \begin{eqnarray}
\label{eq:expectation} 
  \prob\{ S(X, Y) = s \}
%   &=& 
  =
  \frac{1}{n!} \int_{(0,1)^{2n}} \prod_{i=1}^n c(u_{(i)},v_{(s(i))})\, \diff u_i \, \diff v_i
%   \nonumber \\
%   &=& 
  =
  \frac{1}{n!} \expec\left\{ \prod_{i=1}^n c\bigl(\tilde{U}_{(i)}, \tilde{V}_{(s(i))}\bigr)\right\},
% \end{eqnarray}
\end{equation}
% The distribution of $(F(X),G(Y))$, the note following~\eqref{eq:RS}, and 
% the fact that $(0,1)^{2n}\setminus D_n^2$ has Lebesgue measure zero, together 
% yield
% \begin{eqnarray}
% \label{eq:expectation} 
% \prob(S(X,Y)=s)=\prob(S(F(X),G(Y))=s)&=&n!\prob\{R(F(X),G(Y))=(e,s)\},\nonumber
% \\&=&n!\int_{E_s} \prod_{i=1}^n c(u_i,v_i)\, du_idv_i,\nonumber \\
% &=&
% \frac{1}{n!} \int_{(0,1)^{2n}} \prod_{i=1}^n c(u_{(i)},v_{(s(i))})\, du_idv_i,\nonumber \\
% &=& \frac{1}{n!} \expec\left\{ \prod_{i=1}^n 
% c\Big(U_{(i)},V_{(s(i))}\Big)\right\},
% \end{eqnarray}
where $\tilde{U}_1,\ldots,\tilde{U}_n, \tilde{V}_1,\ldots,\tilde{V}_n$ are iid random variables, uniformly distributed on $(0, 1)$. The result then follows from the fact 
that $\tilde{U}_{(i)} = \sum_{j=1}^i W_{1,j}$ and $\tilde{V}_{(s(i))} = \sum_{j=1}^{s(i)} W_{2,j}$, for $i=1,\ldots,n$, where 
$(W_{\ell,1},\ldots,W_{\ell,n+1})$, for $\ell=1,2$, is the vector of $n+1$ spacings on $(0,1)$ 
based on $\tilde{U}_1,\ldots,\tilde{U}_n$ and $\tilde{V}_1,\ldots,\tilde{V}_n$ for $\ell = 1$ and $\ell = 2$, respectively.
\end{proof}

\begin{proof}[Proof of Lemma~\ref{lem:cmpblt}]
Applying~\eqref{eq:rankComp} with $\sigma = (r_x^*)^{-1}$, we find
\begin{align*}
  r_y \in {\cal C}( r_y^*, {\cal M} )
  &\iff
  \rank( r_y(i_1), \ldots, r_y(i_m) ) = r_y^* \\
  &\iff
  \rank( r_y(i_1), \ldots, r_y(i_m) ) \circ (r_x^*)^{-1} = r_y^* \circ (r_x^*)^{-1} \\
  &\iff
  \rank( r_y(i_{(r_x^*)^{-1}(1)}), \ldots, r_y(i_{(r_x^*)^{-1}(m)}) = s^*.
\end{align*}
By definition, ${\cal M}^* = \{ i_1^*, \ldots, i_m^* \}$, where $i_1^* < \ldots < i_m^*$ are the order statistics of the vector $(r_x(i_1), \ldots, r_x(i_m))$. Since $\rank(r_x(i_1), \ldots, r_x(i_m)) = r_x^*$, we find
\[
  i_{(r_x^*)^{-1}(j)} = r_x^{-1}(i_j^*), \qquad j = 1, \ldots, m.
\]
We obtain that
\begin{align*}
  r_y \in {\cal C}( r_y^*, {\cal M} )
  &\iff
  \rank( r_y \circ r_x^{-1}(i_1^*), \ldots, r_y \circ r_x^{-1}(i_m^*) ) = s^* \\
  &\iff
  r_y \circ r_x^{-1} \in {\cal C}( s^*, {\cal M}^* ). %\qedhere
\end{align*}
\end{proof}

\begin{proof}[Proof of Theorem~\ref{thm:FGM}]
By \eqref{eq:RL:integral}, we have
\[
 \prob_{\theta}(S=s)
 =
 n! \int_{E_s} \prod_{i=1}^n c_{\theta}(u_i,v_i)\, \diff u_i \, \diff v_i 
 =
 n! \int_{E_s} \prod_{i=1}^n \{1+\theta(2u_i-1)(2v_i-1)\} \, \diff u_i \, \diff v_i,
\]
where $E_s=\{(u,v)\in (0,1)^{2n}:u_1<\cdots<u_n, v_{s^{-1}(1)}<\cdots<v_{s^{-1}(n)}\}$. Let 
% \[
$  \Delta = \{w \in (0,1)^n : w_1 + \cdots + w_n < 1 \} $
% \] 
be the standard $n$-dimensional simplex. Consider the following change of variables from $E_s$ to $\Delta^2$:
\begin{equation*}
w_{1i}=
\begin{cases}
u_1
&\text{for } i=1,\\
u_i-u_{i-1}
&\text{for } 2\leq i \leq n,
\end{cases}
\text{ and}\quad
w_{2i}=
\begin{cases}
v_{s^{-1}(1)}
&\text{for } i=1,\\
v_{s^{-1}(i)}-v_{s^{-1}(i-1)}
&\text{for } 2\leq i \leq n.
\end{cases}
\end{equation*}
If we put $w_{\ell,n+1}=1-\sum_{j=1}^nw_{\ell j}$, for $\ell=1,2$, then we can write 
\begin{align*}
  2u_i-1
  &=
  u_i-(1-u_i)
  =
  \sum_{k=1}^i w_{1k} - \sum_{k=i+1}^{n+1} w_{1k}
  =
  \sum_{k=1}^{n+1} (-1)^{\indic(k>i)} w_{1k}, \\
  2v_i-1
  &=
  v_i - (1-v_i)
  =
  \sum_{k=1}^{s(i)} w_{2k} -\sum_{k=s(i)+1}^{n+1} w_{2k}
  =
  \sum_{k=1}^{n+1} (-1)^{\indic(k>s(i))} w_{2k}.
\end{align*}
We obtain
\begin{eqnarray*}
\label{eq:pol}
  \prod_{i=1}^n c_{\theta}(u_i(w_1),v_i(w_2))
  &=&
  \prod_{i=1}^n 
  \left\{ 
    1 + \theta
    \left(\sum_{k=1}^{n+1} (-1)^{\indic(k>i)}w_{1k}\right)
    \left(\sum_{k=1}^{n+1} (-1)^{\indic(k>s(i))}w_{2k}\right) 
  \right\} \nonumber \\
  &=& 
  1+\sum_{j=1}^n \theta^j 
  \sum_{1\leq i_1<i_2<\cdots<i_j\leq n}
   d_j(i_1,\ldots,i_j;w_1) \, d_j(s(i_1),\ldots,s(i_j);w_2),
\end{eqnarray*}
where, for all $w \in \Delta$, 
\begin{eqnarray*}
d_j(i_1,\ldots,i_j;w) 
&=& 
\sum_{k_1=1}^{n+1} \cdots \sum_{k_j=1}^{n+1} (-1)^{\sum_{\ell=1}^j\indic(k_{\ell}>i_{\ell})} w_{k_1} \cdots w_{k_j}
\\
&=&
=
\sum_{k_1=1}^{n+1} \cdots \sum_{k_j=1}^{n+1} (-1)^{\sum_{\ell=1}^j\indic(k_{\ell}>i_{\ell})}w_1^{\sum_{\ell=1}^j \indic(k_\ell=1)} \cdots w_{n+1}^{\sum_{\ell=1}^j \indic(k_\ell=n+1)}. 
\end{eqnarray*}
The equality
\[
  \prob_{\theta}(S=s)
  =
  n! \int_{\Delta} \int_{\Delta} 
  \left\{ \prod_{i=1}^n c_{\theta}(u_i(w_1),v_i(w_2)) \right\}
  \, \diff w_1 \, \diff w_2,
\]
shows that $c_0(s)=1/n!$; indeed, $\int_{\Delta} \diff w = 1/n!$. For $j=1,\ldots,n$, the expression for the coefficient $c_j(s)$ is obtained via
\[
% \begin{eqnarray*}
  d_j(i_1,\ldots,i_j)
%   &=&
  =
  \int_{\Delta} d_j(i_1,\ldots,i_j;w) \, \diff w
%   \\
%  &=&
  =
 \sum_{k_1=1}^{n+1} \cdots \sum_{k_j=1}^{n+1} (-1)^{\sum_{\ell=1}^j\indic(k_{\ell}>i_{\ell})}\frac{\prod_{p=1}^{n+1}\left\{\sum_{\ell=1}^j\indic(k_{\ell}=p)\right\}!}{(n+j)!},
% \end{eqnarray*}
\]
for $\{ i_1, \ldots, i_j \} \subset \{1, \ldots, n\}$, an expression which is seen by recognizing the Dirichlet density normalizing constants. This gives~\eqref{eq:dj}, and~\eqref{eq:cj} follows. 

Finally, to see that $c_n(s) = 0$ for all $s \in \Sgrp_n$, note that
\[
  1
  =
  \sum_{s \in \Sgrp_n} \prob_{\theta}(S(U,V)=s)
  =
  \sum_{j=0}^{n}
  \left\{ \sum_{s \in \Sgrp_n} c_j(s) \right\} \theta^j
  = 
  1 + \sum_{j=1}^{n} \left\{ \sum_{s \in \Sgrp_n} c_j(s) \right\} \theta^j.
\]
It follows that $\sum_{s \in \Sgrp_n} c_j(s) = 0$ for $j = 1,\ldots,n$. Since $c_n(s) = n! \{d_n(1,\ldots,n)\}^2$ does not depend on $s$, we conclude that $c_n(s) = 0$, as required.
\end{proof}

\begin{proof}[Proof of Lemma~\ref{lem:symmetries}]
First, the FGM copula density $c_{\theta}$ is symmetric in its arguments and so 
\[
 \expec\left\{ \prod_{i=1}^n 
c_{\theta}\Big(U_{(i)},V_{(s^{-1}(i))}\Big)\right\}=\expec\left\{ 
\prod_{i=1}^n 
c_{\theta}\Big(U_{(s(i))},V_{(i)}\Big)\right\}
= \expec\left\{ 
\prod_{i=1}^n c_{\theta}\Big(U_{(i)},V_{(s(i))}\Big)\right\}.
\]
By using equation~\eqref{eq:expectation}, we find $\prob_{\theta}(S = s^{-1}) = \prob_{\theta}(S = s)$. 

Second, the FGM family also has the property that $c_{\theta}(u,v) = c_{-\theta}(1-u,v)$ for all $(u,v) \in [0,1]^2$. We find
\[
 \expec\left\{ \prod_{i=1}^n 
c_{\theta}\Big(U_{(i)},V_{(s(i))}\Big)\right\}=\expec\left\{ \prod_{i=1}^n 
c_{-\theta}\Big(U_{(i)},1-V_{(s(i))}\Big)\right\}= \expec\left\{ 
\prod_{i=1}^n c_{-\theta}\Big(U_{(i)},V_{(a\circ s(i))}\Big)\right\},
\]
and so, again by \eqref{eq:expectation}, $\prob_{\theta}(S=s)=\prob_{-\theta}(S=a\circ s)$.

These two results imply $\prob_{-\theta}(S=s\circ a) = \prob_{-\theta}( S = a \circ s^{-1} ) = \prob_\theta(S = s^{-1}) = \prob_\theta(S = s)$. Equation~\eqref{eq:symmetries} follows.
\end{proof}

\begin{proof}[Proof of Theorem~\ref{thm:mode}]
Recall the expression for $\prob_\theta(S = s) = \sum_{j=0}^{n-1} c_j(s) \, \theta^j$ for $s \in \Sgrp_n$ in Theorem~\ref{thm:FGM}. Since $\prob_{\theta}(S=s)=\prob_{-\theta}(S=a\circ s)$, for $\theta \in [-1,1]$, we have the relation
\begin{equation}
\label{eq:coeffs}
 c_j(a \circ s)=
\begin{cases}
c_j(s),
&\text{if $j$ is even},\\
-c_j(s), &\text{if $j$ is odd}.
\end{cases}
\end{equation}
For every permutation $s \in \Sgrp_n$ and every $j = 1, \ldots, n$, we have the identity $\{ \{i_1,\ldots,i_j\} : 1 \leq i_1 < i_2 < \ldots < i_j \leq n\} = \{ \{s(i_1),\ldots,s(i_j)\} : 1 \leq i_1 < i_2 < \ldots < i_j \leq n\}$. Moreover, the expression $d_j$ in \eqref{eq:dj} is symmetric in its arguments, that is,
$d_j(i_1,\ldots,i_j) = d_j(i_{\sigma(1)},\ldots,i_{\sigma(j)})$ for $\sigma \in 
\Sgrp_j$. It follows that
\[
  \sum_{1\leq i_1<i_2<\cdots<i_j\leq n} \{d_j(i_1,\ldots,i_j)\}^2 
  = 
  \sum_{1\leq i_1<i_2<\cdots<i_j\leq n} \{d_j(s(i_1),\ldots,s(i_j))\}^2.
\]
By the Cauchy--Schwarz inequality, $\lvert c_j(s) \rvert \leq c_j(e)$ for $j=0,\ldots,n-1$. Finally, by the triangle inequality,
\begin{eqnarray*}
 \prob(S=s)=\expec_{\pi}\{\prob_{\theta}(S=s)\}&\leq& 
\sum_{j=0}^{n-1} \lvert c_j(s)\expec_{\pi}(\theta^j) \rvert
% \\
% &\leq& 
\le
\sum_{k=0}^{\lfloor n/2
\rfloor-1} c_{2k+1}(e) \, \lvert\expec_{ 
\pi}(\theta^{2k+1})\rvert + \sum_{k=0}^{\lfloor(n-1)/2
\rfloor} c_{2k}(e)\expec_{\pi}(\theta^{2k})\\
&=&
\begin{cases}
\prob(S=e),
&\text{if all odd order moments are nonnegative},\\
\prob(S=a), &\text{if all odd order moments are nonpositive},
\end{cases}
\end{eqnarray*}
where the last equality follows from~\eqref{eq:coeffs}, with $s=e$.
\end{proof}

\begin{proof}[Proof of Corollary~\ref{cor:mode}]
By Theorem~\ref{thm:mode}, we need to look at the signs of the odd order moments. Let 
$1\leq k \leq n-1$ be an odd integer. We have
\[
  \expec_{\pi_{\alpha,\beta}}(\theta^k)
  =
  \frac{B(\beta,\beta)}{B(\alpha,\beta)}
  \expec[(2X-1)^k \{X^{\alpha-\beta}-(1/2)^{\alpha-\beta}\}], 
  \qquad X \sim \operatorname{Beta}(\beta,\beta).
\]
This is nonnegative if $\alpha \geq \beta$ and nonpositive if $\alpha \leq \beta$.
\end{proof}

\begin{proof}[Proof of Lemma~\ref{lem:compat}]
Let $S_0, S_1, S_2, \ldots$ be the Markov chain constructed by the algorithm. 

First, two states $s, r \in {\cal C}(s^*, {\cal M}^*)$ are equal if and only if $s(t)=r(t)$, for every 
$t\in {\cal N} \setminus \mathcal{M}^*$. If $s=r$, then 
$\prob(S_{2}=r\mid S_0=s)>0$. Otherwise, let 
${\cal N} \setminus \mathcal{M}^*=\{t_1,\ldots,t_{n-m}\}$, with $1\leq t_1<t_2<\cdots<t_{n-m} \leq 
n$. Now if $s(t_1) \neq r(t_1)$, 
then there are two possible cases; either $r(t_1)=s(t_k)$ for some 
$k\in\{2,\ldots,n-m\}$, and a call to move 
$\Move_{1}$ can generate %(with positive probability) 
$S_1$ with 
$S_1(t_1)=r(t_1)$. Or, in the second case, 
$r(t_1)=s(i_{\ell}^*)$ for some $\ell \in \{1,\ldots,m\}$, and then a call to 
move $\Move_{2}$ can generate $S_1$ with 
$S_1(t_1)=r(t_1)$. Continuing this way for $t_2,\ldots,t_{n-m}$, we see that 
$\prob(S_{k}=r\mid S_0=s)>0$ for some $k = 1, \ldots, n-m$, for every $s$ and $r$ in ${\cal C}(s^*, {\cal M}^*)$.

To show aperiodicity in the case where $1 < m < n$, take $s \in {\cal C}(s^*, {\cal M}^*)$, and let $i_{j_1}^*, i_{j_2}^* \in {\cal M}^*$ and $t_1 
\in {\cal N} \setminus {\cal M}^*$. Three successive calls 
to move $\Move_{2}$ can generate $S_1$, $S_2$, and $S_3$ with $S_1(t_1) = s(i_{j_1}^*)$, 
$S_2(t_1) = s(i_{j_2}^*)$, and %finally 
$S_3(t_1) = s(t_1)$. Therefore $\prob(S_2 = s \mid S_0 = s) \wedge \prob(S_3 = s \mid S_0 = s) > 0$.%, and so the chain is aperiodic. 

By irreducibility and aperiodicity, the above Markov chain has a unique stationary distribution on ${\cal C}(s^*, {\cal M}^*)$. By symmetry of the transition kernel, i.e., $\prob(S_1 = s \mid S_0 = r) = \prob(S_1 = r \mid S_0 = s)$ for all $s,r \in {\cal C}(s^*, {\cal M}^*)$, it follows that this stationary distribution must be the uniform one.
\end{proof}

\begin{proof}[Proof of Lemma~\ref{thm:instrumental}]
We have $q_W(w\mid w_0)=\int_0^1q_{W,\Lambda}(w,\lambda \mid w_0)\, \diff\lambda$, with 
\[
% \begin{eqnarray*} 
  q_{W,\Lambda}(w,\lambda\mid w_0)
%   &=&
  =
  n! \, \indic_{\Delta} \left( \frac{w-(1-\lambda)w_0}{\lambda} \right) \, \lambda^{-n} \, g(\lambda) \\
%   &=& 
  =
  n! \, \indic_\Delta(w) \, \lambda^{-n} \, g(\lambda) \, \indic_{(1-\delta(w;w_0),1)}(\lambda),
% \end{eqnarray*}
\]
for $\lambda \in (0,1)$ and $w \in \Delta$.
\end{proof}

\bibliographystyle{biometrika}
\bibliography{biblio}%
\end{document}